%% file: arxiv.tex
\documentclass{article}

\usepackage{microtype}
\usepackage{graphicx}
\usepackage{subfigure}
\usepackage{booktabs} 

\usepackage{algorithmic}
\usepackage[utf8]{inputenc}
\usepackage{fullpage}
\usepackage[round]{natbib}
\usepackage{amsfonts,amssymb,amsmath,amsthm}
\usepackage{algorithm}
\usepackage[colorlinks=true,allcolors=blue,hyperfootnotes=false]{hyperref}
\usepackage{enumitem}
\usepackage[italic,defaultmathsizes]{mathastext}

\usepackage{hyperref}
\usepackage[capitalise,nameinlink]{cleveref}

\include{header}

\title{Near-optimal Regret Bounds for Stochastic Shortest Path}
\author{
    Alon Cohen
    \thanks{Google Research, Tel Aviv; \href{mailto:aloncohen@google.com}{\texttt{aloncohen@google.com}}.}
    \and
    Haim Kaplan
    \thanks{Tel-Aviv University and Google Research, Tel Aviv; \href{mailto:haimk@post.tau.ac.il}{\texttt{haimk@post.tau.ac.il}}.}
    \and
    Yishay Mansour
    \thanks{Tel-Aviv University and Google Research, Tel Aviv; Supported in part by a grant from the ISF. \href{mailto:mansour@tau.ac.il}{\texttt{mansour@tau.ac.il}}.}
    \and
    Aviv Rosenberg
    \thanks{Tel-Aviv University; \href{mailto:avivros007@gmail.com}{\texttt{avivros007@gmail.com}}.}
}

\begin{document}
\maketitle

\input{main}

\bibliographystyle{plainnat}
\bibliography{bib}

\input{appendix}

\end{document}

%% file: header.tex
\newtheorem{theorem}{Theorem}[section]
\newtheorem{lemma}[theorem]{Lemma}
\newtheorem{corollary}[theorem]{Corollary}

\newtheorem*{lemma*}{Lemma}
\newtheorem*{theorem*}{Theorem}

\theoremstyle{definition}
\newtheorem{assumption}{Assumption}
\newtheorem{definition}{Definition}
\newtheorem{observation}[theorem]{Observation}

\newcommand{\regret}{R_K}
\newcommand{\tregret}[1]{\wt R_{#1}}
\newcommand{\bbE}{\mathbb{E}}
\newcommand{\bbR}{\mathbb{R}}
\newcommand{\bbV}{\mathbb{V}}
\newcommand{\calF}{\mathcal{F}}

\newcommand{\wt}{\widetilde}

\newcommand{\tO}{\wt O}

\renewcommand{\Pr}{\mathbb{P}}
\newcommand{\ind}{\mathbb{I}}
\newcommand{\TV}[2]{\text{TV}(#1, #2)}
\newcommand{\KL}[2]{\text{KL}(#1 \;\|\; #2)}

\newcommand{\sinit}{s_\text{init}}
\newcommand{\ssink}{g}
\newcommand{\ctg}[1]{J^{#1}}
\newcommand{\ctgopt}{\ctg{\piopt}}
\newcommand{\optimisticctg}[1]{\wt{J}^{#1}}
\newcommand{\costbound}{B_\star}
\newcommand{\costboundestimate}{\wt{B}}
\newcommand{\timebound}{T_\star}
\newcommand{\piopt}{\pi^\star}
\newcommand{\policytime}[1]{T^{#1}}

\newcommand{\timeopt}{\policytime{\piopt}}
\newcommand{\propset}{\Pi_\text{proper}}
\newcommand{\geventi}[1]{\Omega^{#1}}

\newcommand{\indgeventi}[1]{\ind \{ \geventi{#1} \}}
\newcommand{\numintervals}{M}
\newcommand{\numepochs}{E}
\newcommand{\cmin}{c_{\text{min}}}
\newcommand{\highprobcostbound}[1]{24 \costbound \log \frac{4 #1}{\delta}}
\newcommand{\highprobcostboundestimate}[1]{24 \costboundestimate \log \frac{4 #1}{\delta}}

\newcommand{\totaltime}{T}

\newcommand{\Hm}{H^m}
\newcommand{\Ik}{I^k}

\newcommand{\numvisitsuntilknownhoff}{\frac{5000 \costbound^2 |S|}{\cmin^2}
\log \frac{\costbound |S| |A|}{\delta \cmin} }

\newcommand{\numvisitsuntilknownhoffest}{\frac{5000 \costboundestimate^2 |S|}{\cmin^2}
\log \frac{\costboundestimate |S| |A|}{\delta \cmin} }

\newcommand{\numvisitsuntilknownbern}{\frac{\costbound |S|}{\cmin} \log \frac{\costbound |S| |A|}{\delta \cmin}}

\newcommand{\cshoff}[3]{5 \sqrt{ \frac{ |S| \log \bigl(|S| |A| N_+^{#1}(#2,#3) / \delta \bigr)}{N_+^{#1}(#2,#3)}}}

\newcommand{\csbern}[4]{8 \sqrt{\frac{P(#4 \mid #2,#3) \log \bigl(|S| |A| N_+^{#1}(#2,#3) / \delta \bigr) }{N_+^{#1}(#2,#3)}}
+ 
\frac{136 \log \bigl(|S| |A| N_+^{#1}(#2,#3) / \delta \bigr)}{N_+^{#1}(#2,#3)}}

\newcommand{\traj}[1]{U^{#1}}
\newcommand{\trajconcat}[1]{\bar U^{#1}}

\newcommand{\logfactor}{L}
\newcommand{\logfactort}{\wt L}
\newcommand{\logfactorcminfull}{\log (K \costbound |S| |A| / \delta \cmin)}
\newcommand{\logfactortfull}{\log(K \costbound \timebound |S| |A|/\delta)}

\DeclareMathOperator*{\argmin}{arg\,min}

%% file: main.tex
\begin{abstract}
    Stochastic shortest path (SSP) is a well-known problem in planning and control, in which an agent has to reach a goal state in minimum total expected cost.
    In the learning formulation of the problem, the agent is unaware of the environment dynamics (i.e., the transition function) and has to repeatedly play for a given number of episodes, while learning  the problem's optimal solution. 
    Unlike other well-studied models in reinforcement learning (RL), the length of an episode is not predetermined (or bounded) and is influenced by the agent's actions. 
    Recently, \citet{tarbouriech2019noregret} studied this problem in the context of regret minimization, and provided an algorithm whose regret bound is inversely proportional to the square root of the minimum instantaneous cost.
    In this work we  remove this dependence on the minimum cost---we give an algorithm that guarantees a regret bound of $\widetilde{O}(\costbound |S| \sqrt{|A| K})$, where $\costbound$ is an upper bound on the expected cost of the optimal policy, $S$ is the set of states, $A$ is the set of actions and $K$  is the number of episodes. We additionally show that any learning algorithm must have at least $\Omega(\costbound \sqrt{|S| |A| K})$ regret in the worst case.
\end{abstract}

\section{Introduction}

Stochastic shortest path (SSP) is one of the most basic models in reinforcement learning (RL). 
It includes the discounted return model and the finite-horizon model as special cases. 
In SSP the goal of the agent is to reach a predefined goal state in minimum expected cost. 
This setting captures a wide variety of realistic scenarios, such as car navigation,  game playing and drone flying;
i.e., tasks carried out in episodes that eventually terminate.

%

The focus of this work is on regret minimization in SSP. It builds on extensive literature on theoretical aspects of online RL, and in particular on the copious works about regret minimization in either the average cost model or the finite-horizon model.
A major contribution to this literature is the UCRL2 algorithm \cite{AuerUCRL} that gives a general framework to achieve optimism in face of uncertainty for these settings. 
The main methodology is to define a confidence set that includes the true model parameters with high probability. The algorithm periodically computes an optimistic policy that minimizes the overall expected cost simultaneously over all policies and over all parameters within the confidence set, and proceeds to play this policy.

The only  regret minimization algorithm specifically designed for SSP is that of \citet{tarbouriech2019noregret} that assumes that all costs are bounded away from zero (i.e., there is a $\cmin>0$ such that all costs are in the range $[\cmin,1]$). They show a regret bound that scales as $O(D^{3/2} |S| \sqrt{|A| K/\cmin})$ where $D$ is the minimum expected time of reaching the goal state from any state, $S$ is the set of states, $A$ is the set of actions and $K$ is the number of episodes. 
In addition, they show that the algorithm's regret is $\tO (K^{2/3})$ when the costs are arbitrary (namely, may be zero).

Here we improve upon the work of \citet{tarbouriech2019noregret} in several important aspects.
First, we remove the dependency on $\cmin^{-1}$ and allow for zero costs while maintaining  regret of $\tO(\sqrt{K})$. Second, we give a much simpler algorithm in which the computation of the optimistic policy has a simple solution.
Our main regret term is 
$
    \tO(\costbound |S| \sqrt{|A| K}),
$
where $\costbound$ is an upper bound on the expected cost of the optimal policy (note that $\costbound \le D$). 
We show that this is almost optimal by giving a lower bound of
$
    \Omega(\costbound \sqrt{|S| |A| K}).
$

Our technical contribution is as follows. We start by assuming that the costs are lower bounded by $\cmin$ and give an algorithm that is simple to analyze and achieves a regret bound of $\tO(\costbound^{3/2} |S| \sqrt{|A| K /\cmin})$. Note that this bound is comparable to the one of \citet{tarbouriech2019noregret}, yet our algorithm and its analysis are significantly simpler and more intuitive. 
We subsequently improve our algorithm by utilizing better confidence sets based on the Bernstein concentration inequality \citep{azar2017minimax}. This algorithm is even simpler than our first one mainly since picking the parameters of the optimistic model is particularly easy.
The analysis, however, is somewhat more delicate.
We achieve our final bound by perturbing the instantaneous costs to be at least $\epsilon > 0$. 
The additional cost due to this perturbation has a small effect since
the dependency of our regret on $\cmin^{-1}$ is additive and does not multiply any term depending on $K$. 

\subsection{Related work.}

Early work by \citet{bertsekas1991analysis} studied the problem of planning in SSPs, that is, computing the optimal strategy efficiently in a known SSP instance.
They established that, under certain assumptions, the optimal strategy is a deterministic stationary policy (a mapping from states to actions) and can be computed efficiently using standard planning algorithms, e.g., Value Iteration or Policy Iteration.

The extensive literature about regret minimization in RL focuses on the the average-cost infinite-horizon model \cite{bartlett2009regal,AuerUCRL} and on the finite-horizon model \cite{osband2016generalization,azar2017minimax,dann2017unifying,zanette2019tighter}.
These recent works give algorithms with near-optimal regret bounds using  Bernstein-type concentration bounds.

Another related model is that of loop-free SSP with adversarial costs \cite{neu2010loopfree,neu2012adversarial,zimin2013online,rosenberg2019bandit,rosenberg2019full}.
This model eliminates the challenge of avoiding policies that never terminate, but the adversarial costs pose a different, unrelated, challenge.

\section{Preliminaries and Main Results}

An instance of the SSP problem is a Markov decision process (MDP) $M = (S,A,P,c,\sinit)$ where $S$ is the state space and $A$ is the action space. The agent begins at the initial state $\sinit$, and ends her interaction with $M$ by arriving at the goal state $\ssink$ (where $\ssink\not\in S$).
Whenever she plays action $a$ in state $s$, she pays a cost $c(s,a) \in [0,1]$ and the next state $s' \in S$ is chosen with probability $P(s' \mid s,a)$. 
Note that to simplify the  presentation we avoid addressing the goal state $\ssink$ explicitly -- we assume that the probability of reaching the goal state by playing action $a$ at state $s$   is $1-\sum_{s'\in S}P(s'\mid s,a)$.

We now review planning in a known SSP instance.
Under certain assumptions that we shall briefly discuss, the optimal behaviour of the agent, i.e., the policy that minimizes the expected total cost of reaching the goal state from \emph{any} state, is  a stationary, deterministic and proper policy. 
A stationary and deterministic policy $\pi : S \mapsto A$ is a mapping that selects action $\pi(s)$ whenever the agent is at state $s$.
A proper policy is defined as follows.

\begin{definition}[Proper and Improper Policies]
    A policy $\pi$ is \emph{proper} if playing $\pi$ reaches the goal state with probability $1$ when starting from any state.
    A policy is \emph{improper} if it is not proper.
\end{definition}

Any policy $\pi$ induces a \emph{cost-to-go function} $\ctg{\pi} : S \mapsto [0, \infty]$ defined as 
$
    \ctg{\pi}(s) = \lim_{T \rightarrow \infty} \bbE_\pi \bigl[ \sum_{t=1}^T c(s_t,a_t) \mid s_1 = s \bigr],
$
where the expectation is taken w.r.t the random sequence of states generated by playing according to $\pi$ when the initial state is $s$. For a proper policy $\pi$, since the number of states $|S|$ is finite, it follows that $\ctg{\pi}(s)$ is finite for all $s \in S$. However, note that $\ctg{\pi}(s)$ may be finite even if $\pi$ is improper. We additionally denote by $\policytime{\pi}(s)$ the expected time it takes for $\pi$ to reach $\ssink$ starting at $s$; in particular, if $\pi$ is proper then $\policytime{\pi}(s)$ is finite for all $s$, and if $\pi$ is improper there must exist some $s$ such that $\policytime{\pi}(s) = \infty$.
In this work we assume the following about the SSP model.

\begin{assumption}
    \label{ass:ex-prop}
    There exists at least one proper policy.
\end{assumption}

With \cref{ass:ex-prop}, we have the following important properties of proper policies. In particular, the first result shows that a policy is proper if and only if its cost-to-go function satisfies the Bellman equations. The second result proves that a policy is optimal if and only if it satisfies the Bellman optimality criterion. Note that they assume that every improper policy has high cost.

\begin{lemma}[{\citealp[Lemma 1]{bertsekas1991analysis}}]
\label{lem:bertsekas-proper}
Suppose that \cref{ass:ex-prop} holds and that for every improper policy $\pi'$ there exists at least one state $s \in S$ such that $\ctg{\pi'}(s) = \infty$. 
        Let $\pi$ be any policy, then
        \begin{enumerate}[nosep,label=(\roman*)]
            \item 
                If there exists some $\ctg{} : S \mapsto \bbR$ such that 
                $
                    \ctg{}(s) 
                    \ge 
                    c \bigl(s, \pi(s) \bigr) 
                    + 
                    \sum_{s'\in S} P \bigl(s' \mid s, \pi(s) \bigr) \ctg{}(s')
                $
                for all $s \in S$, then $\pi$ is proper. 
                Moreover, it holds that 
                $
                    \ctg{\pi}(s) \leq \ctg{}(s), \; \forall s \in S.
                $
            \item If $\pi$ is proper then $\ctg{\pi}$ is the unique solution to the equations 
            $
                \ctg{\pi}(s) 
                = 
                c \bigl(s, \pi(s) \bigr) 
                + 
                \sum_{s'\in S} P \bigl(s' \mid s, \pi(s) \bigr) \ctg{\pi}(s')
            $
            for all $s \in S$.
        \end{enumerate}
        
\end{lemma}

\begin{lemma}[{\citealp[Proposition 2]{bertsekas1991analysis}}]
    \label{lem:bertsekas-optimal}
    Under the conditions of \cref{lem:bertsekas-proper}
    the optimal policy $\piopt$ is stationary, deterministic, and proper. Moreover, a policy $\pi$ is optimal if and only if it satisfies the Bellman optimality equations for all $s \in S$:
    \begin{alignat}{2} \label{eq:bellman}
        &\ctg{\pi}(s) 
        &&= 
        \min_{a \in A} c \bigl(s, a \bigr) + \sum_{s'\in S} P \bigl(s' \mid s, a \bigr) \ctg{\pi}(s'), \\
        &\pi(s) 
        &&\in 
        \argmin_{a \in A} c \bigl(s, a \bigr) + \sum_{s'\in S} P \bigl(s' \mid s, a \bigr) \ctg{\pi}(s').
        \nonumber
    \end{alignat}
\end{lemma}

In this work we are not interested in approximating the optimal policy overall, but rather the best \emph{proper} policy. In this case the second requirement in the lemmas above, that for every improper policy $\pi$ there exists some state $s \in S$ such that $\ctg{\pi}(s) = \infty$, can be circumvented in the following way \citep{bertsekas2013stochastic}.
First, note that this requirement is trivially satisfied when all instantaneous costs are strictly positive. 
 Then, one can perturb the instantaneous costs by adding a small positive cost $\epsilon \in [0,1]$, i.e., the new cost function is $c_\epsilon(s,a) = \max\{c(s,a),\epsilon\}$. After this perturbation, all proper policies remain proper, and every improper policy has infinite cost-to-go from some state (as all costs are positive). 
In the modified MDP, we apply \cref{lem:bertsekas-optimal}  and obtain an optimal policy $\piopt_\epsilon$ that is stationary, deterministic and proper and has a cost-to-go function ${\ctg{}}^\star_\epsilon$. Taking the limit as $\epsilon \rightarrow 0$, we have that $\piopt_\epsilon \rightarrow \piopt$ and ${\ctg{}}^\star_\epsilon \rightarrow \ctgopt$, where $\piopt$ is the optimal \emph{proper} policy in the original model that is also stationary and deterministic, and $\ctgopt$ denotes its cost-to-go function. 
We use this observation to obtain  
\cref{corr:hoeffdingbound,corr:bernsteinbound} below that only require \cref{ass:ex-prop} to hold.

\paragraph{Learning formulation.}

We assume that the costs are deterministic and known to the learner, and the transition probabilities $P$ are fixed but unknown to the learner.
The learner interacts with the model in episodes: each episode starts at the initial state $\sinit$, and ends when the learner reaches the goal state $\ssink$ (note that she might \emph{never} reach the goal state). 
Success is measured by the learner's regret over $K$ such episodes, that is the difference between her total cost over the $K$ episodes and the total expected cost of the optimal proper policy:
\[
    \regret 
    = 
 \sum_{k=1}^K\sum_{i=1}^{\Ik} c(s^k_i,a^k_i) 
    - 
    K \cdot \min_{\pi \in \propset} \ctg{\pi} (\sinit),
\]
where $\Ik$ is the time it takes the learner to complete episode $k$ (which may be infinite),
$\propset$ is the set of all stationary, deterministic and proper policies (that is not empty by \cref{ass:ex-prop}), and $(s_i^k,a^k_i)$ is the $i$-th state-action pair at episode $k$. 
In the case that $\Ik$ is infinite for some $k$, we define $\regret = \infty$.

We denote the optimal proper policy by $\piopt$, i.e., $\ctgopt(s) = \argmin_{\pi \in \propset} \ctg{\pi}(s)$ for all $s \in S$. 
Moreover, let $\costbound > 0$ be an upper bound on the values of $\ctgopt$ and let $\timebound > 0$ be an upper bound on the times $\timeopt$, i.e., $\costbound \geq \max_{s \in S} \ctgopt (s)$ and $\timebound \geq \max_{s \in S} \timeopt (s)$.



\subsection{Summary of our results}

In \cref{sec:algorithm} we present our Hoeffding-based algorithms (\cref{alg:knownb,alg:unkownb}) and their analysis. 
In \cref{sec:bernstein} we show our Bernstein-based algorithm (\cref{alg:bern-o-ssp}) for which we prove improved regret bounds. 
In addition, we give a lower bound on the learner's regret showing that \cref{alg:bern-o-ssp} is near-optimal (see \cref{sec:lowerbound}). 

The learner must reach the goal state otherwise she has infinite regret. Therefore, she has to trade-off two objectives, one is to reach the goal state and the other is to minimize the cost. Under the following assumption, the two objectives essentially coincide.


\begin{assumption}
    \label{ass:c-min}
    All costs are positive, i.e., there exists $\cmin > 0$ such that $c(s,a) \geq \cmin$ for every $(s,a) \in S \times A$.
\end{assumption}

This assumption allows us to upper bound the running time of the algorithm by its total cost up to a factor of $\cmin^{-1}$.
In particular, it guarantees that any policy that does not reach the goal state has infinite cost, so any bounded regret algorithm has to reach the goal state.
We eventually relax \cref{ass:c-min} by a technique similar to that of  \citet{bertsekas2013stochastic}. We add a small positive perturbation to the instantaneous costs and run our algorithms on the model with the perturbed costs.
This provides a regret bound that scales with the expected running time of the optimal policy. 

We now summarize our results.
For ease of comparison, we first present our regret bounds for both the Hoeffding and Bernstein-based algorithms for when \cref{ass:c-min} holds, and subsequently show the regret bounds of both algorithms for the general case.
%
%
In order to simplify the presentation of our results, we assume that $|S| \ge 2$, $|A| \ge 2$ and $K \ge |S|^2 |A|$ throughout. 
In addition, we denote $\logfactor = \logfactorcminfull$.
The complete proof of all statements is found in the supplementary material.

\paragraph{Positive costs.} 
The following results hold when \cref{ass:c-min} holds (recall that we always assume \cref{ass:ex-prop}). In particular, when this assumption holds the optimal policy overall is  proper (\cref{lem:bertsekas-optimal}) hence the regret bounds below are with respect  to the best overall policy.

\begin{theorem}
    \label{thm:reg-bound-unknownb}
    Suppose that \cref{ass:c-min} holds.
    With probability at least $1 - \delta$ the regret of \cref{alg:unkownb} is bounded as follows:
    \begin{align*}
        \regret
        &=
        O \Biggl( 
        \sqrt{\frac{\costbound^3 |S|^2 |A| K}{\cmin}} \logfactor
        + 
        \frac{\costbound^3 |S|^2 |A|}{\cmin^2} \logfactor^2
         \Biggr).
    \end{align*}
\end{theorem}

The main issue with the regret bound in \cref{thm:reg-bound-unknownb} is that it scales with $\sqrt{K / \cmin}$ which cannot be avoided regardless of how large $K$ is with respect to $\cmin^{-1}$. This problem is alleviated in \cref{alg:bern-o-ssp} that uses the tighter Bernstein-based confidence bounds.

\smallskip
\begin{theorem}
\label{thm:bern-reg-bound}
    Assume  that \cref{ass:c-min} holds. 
    With probability at least $1 - \delta$ the regret of \cref{alg:bern-o-ssp} is bounded as follows:
    \begin{align*}
        \regret
        &=
        O \Biggl(
        \costbound |S| \sqrt{|A| K} \logfactor
        + 
        \sqrt{\frac{\costbound^3 |S|^4 |A|^2}{\cmin}} \logfactor^2
         \Biggr).
    \end{align*}
\end{theorem}

Note that when $K \gg \costbound |S|^2 |A| / \cmin$, the regret bound above scales as $\tO(\costbound |S| \sqrt{|A| K})$ thus obtaining a near-optimal rate.

{\bf Arbitrary costs.} Recall that in this case we can no longer assume that the optimal policy is proper. Therefore, the regret bounds below are with comparison to the best \emph{proper} policy. 
\cref{ass:c-min} can be easily alleviated by adding a small fixed cost to the cost of all state-action pairs. Following the perturbation of the costs, we obtain regret bounds from \cref{thm:reg-bound-unknownb,thm:bern-reg-bound} with $\cmin \gets \epsilon$ and $\costbound \gets \costbound + \epsilon \timebound$, and the learner also suffers an additional cost of $\epsilon \timebound K$ due to the misspecification of the model caused by the perturbation. By picking
$\epsilon$ to balance these terms we get the following corollaries (letting $\logfactort = \logfactortfull$).

\begin{corollary} \label{corr:hoeffdingbound} 
    Running \cref{alg:unkownb} using costs $c_\epsilon(s,a) = \max\{c(s,a),\epsilon\}$ for $\epsilon = (|S|^2 |A| / K)^{1/3}$ gives the following regret bound with probability at least $1-\delta$:
    \begin{align*}
        \regret 
        & =
        O \Biggl(
        \timebound^3 |S|^{2/3} |A|^{1/3} K^{2/3} \logfactort
        +
        \timebound^3 |S|^2 |A| \logfactort^2
        \Biggr).
    \end{align*}
\end{corollary}

\begin{corollary}\label{corr:bernsteinbound} 
    Running \cref{alg:bern-o-ssp} using costs $c_\epsilon(s,a) = \max\{c(s,a),\epsilon\}$ for $\epsilon = |S|^2 |A| / K$ gives the following regret bound with probability at least $1-\delta$:
    \begin{align*}
        \regret 
        &=
        O \Biggl(
        \costbound^{3 / 2} |S| \sqrt{|A| K} \logfactort
        +
        \timebound^{3 / 2} |S|^2 |A| \logfactort^2
        \Biggr).
    \end{align*}
    Moreover, when the algorithm knows $\costbound$ and $K \gg |S|^2 |A| \timebound^2$, then choosing $\epsilon = \costbound |S|^2 |A| / K$ gets a near-optimal regret bound of $\tO(\costbound |S| \sqrt{|A| K})$.
\end{corollary}

{\bf Lower bound.}
In \cref{sec:lowerbound} we show that \cref{corr:bernsteinbound} is nearly-tight using the following theorem.

\begin{theorem}
    \label{thm:lowerbound}
    There exists an SSP problem instance $M = (S, A, P, c, \sinit)$ in which $\ctgopt(s) \le \costbound$ for all $s \in S$, $|S| \ge 2$, $|A| \ge 16$, $\costbound \ge 2$, $K \ge |S| |A|$, and $c(s,a)=1$ for all $s \in S, a \in A$, such the expected regret of any learner after $K$ episodes satisfies
    \[
        \bbE[\regret] \ge \frac{1}{1024} \costbound \sqrt{|S| |A| K}.
    \]
\end{theorem}


\section{Hoeffding-type Confidence Bounds} \label{sec:algorithm}

We start with a simpler case in which $\costbound$ is known to the learner. 
In \cref{sec:unkownb} we alleviate this assumption with a penalty of an additional log-factor in the regret bound.
For now, we prove the following bound on the learner's regret.

\begin{theorem}
\label{thm:reg-bound-knownb}
    Suppose that \cref{ass:c-min} holds.
    With probability at least $1 - \delta$ the regret of \cref{alg:knownb} is bounded as follows:
    \begin{align*}
        \regret
        &=
        O \Biggl( \sqrt{ \frac{ \costbound^3 |S|^2 |A| K}{\cmin}} \logfactor
        + 
        \frac{\costbound^3 |S|^2 |A|}{\cmin^2} \logfactor^{3/2}
        \Biggr).
    \end{align*}
\end{theorem}

Our algorithm follows the known concept of optimism in face of uncertainty. 
That is, it maintains confidence sets that contain the true transition function with high probability and picks an optimistic optimal policy---a policy that minimizes the expected cost over all policies and all transition functions in the current confidence set.
The computation of the optimistic optimal policy can be done efficiently as shown by \citet{tarbouriech2019noregret}.
Construct an augmented MDP whose states are $S$ and its action set consists of tuples $(a, \wt P)$ where $a \in A$ and $\wt P$ is any transition function such that

\begin{equation} \label{eq:hoff-confidence-set}
    \bigl\| \wt P (\cdot | s,a) - \Bar{P} (\cdot | s,a) \bigr\|_1 
    \leq 
    5 \sqrt{\frac{|S| \log (|S| |A| N_+(s,a) / \delta )}{N_+(s,a)}}
\end{equation}

where $\Bar{P}$ is the empirical estimate of $P$.
It can be shown that the optimistic policy and the optimistic model, i.e., those that minimize the expected total cost over all policies and feasible transition functions, correspond to the optimal policy of the augmented MDP. 

\begin{algorithm}[t]
    \caption{\sc Hoeffding-type confidence bounds and known $\costbound$}
    \label{alg:knownb}
    \begin{algorithmic} 
        
        \STATE {\bfseries input:} state space $S$, action space $A$, bound on cost-to-go of optimal policy $\costbound$, confidence parameter $\delta$.
        
         \STATE {\bfseries initialization:} $\forall (s,a,s') \in S \times A \times S: N(s,a,s') \leftarrow 0 , N(s,a) \leftarrow 0$, an arbitrary policy $\tilde{\pi}$, $t \leftarrow 1$.
         
         \FOR{$k=1,2,\ldots$}
         
         \STATE set $s_t \leftarrow \sinit$.
        
        \WHILE{$s_t \neq \ssink$}
        
        \STATE follow optimistic optimal policy:
         $a_t \leftarrow \Tilde{\pi}(s_t)$.
         
        \STATE observe next state $s_{t+1} \sim P(\cdot \mid s_t,a_t)$.
         
        \STATE {\bf update}:
        $N(s_t,a_t,s_{t+1}) \leftarrow N(s_t,a_t,s_{t+1})+1$, 
        $N(s_t,a_t) \leftarrow N(s_t,a_t)+1$.
         
        \IF{$N(s_{t+1}, \tilde{\pi}(s_{t+1})) \!\le\! \numvisitsuntilknownhoff$ \!or\! $s_{t+1} \!=\! \ssink$}
            
            \STATE \# start new interval
             
            \STATE {\bf compute} empirical transition function $\Bar{P}$ as $\Bar{P} (s' | s,a) = N(s,a,s') / N_+(s,a)$ where $N_+(s,a) = \max\{N(s,a),1\}$.
             
            \STATE {\bf compute} optimistic policy $\tilde{\pi}$ by minimizing expected cost over transition functions $\wt P$ that satisfy~\cref{eq:hoff-confidence-set}.
        \ENDIF
        \STATE {\bf set} $t \leftarrow t+1$.
        \ENDWHILE
        \ENDFOR
    \end{algorithmic}
\end{algorithm}

To ensure that the algorithm reaches the goal state in every episode, we define a state-action pair $(s,a)$ as {\em known} if the number of visits to this pair is at least $\numvisitsuntilknownhoff$ and as {\em unknown} otherwise.
We show with high probability the optimistic policy chosen by the algorithm will be proper once all state-action pairs are known.
However, when some pairs are still unknown, our chosen policies may be improper.
This implies that the strategy of keeping the policy fixed throughout an episode, as done usually in episodic RL, will fail.
Consequently, our algorithm changes policies at the start of every episode and also every time we reach an unknown state-action pair.

Formally, we split the time into {\em intervals}.
The first interval begins at the first time step, and every interval ends by reaching the goal state or a state $s$ such that $(s, \tilde{\pi}(s))$ is unknown (where $\tilde{\pi}$ is the current policy followed by the learner).
Recall that once all state-action pairs are known, the optimistic policy will eventually reach the goal state. 
Therefore, recomputing the optimistic policy at the end of every interval ensures that the algorithm will eventually reach the goal state with high probability. 
Note that the total number of intervals is at most the number of visits to an unknown state-action pair plus the number of episodes.

\begin{observation}
\label{obs:num-intervals}
The total number of intervals, $\numintervals$, is 
\[
    O \biggl( K + \frac{\costbound^2 |S|^2 |A|}{\cmin^2} \log \frac{ \costbound |S| |A|}{\delta \cmin} \biggr).
\]
\end{observation}

\subsection{Analysis} \label{sec:analysis}
The proof of \cref{thm:reg-bound-knownb} begins by defining the ``good event'' in which our confidence sets contain the true transition function and the total cost in every interval is bounded. This in turn implies that all episodes end in finite time. We prove that the good event holds with high probability.

Then, independently, we give a high-probability bound on the regret of the algorithm when the good event holds. 
To do so, recall that at the beginning of every interval $m$, the learner computes an optimistic policy by minimizing over all policies and over all transition functions within the current confidence set. We denote the chosen policy by  $\tilde{\pi}^m$ and let $\wt P_m$ be the minimizing transition function (i.e., the optimistic model). 
A key observation is that by
the definition of our confidence sets, $\wt P_m$ is such that there is always some positive probability to transition to the goal state directly from any state-action. This implies that all policies are proper in the optimistic model and that the cost-to-go function of $\tilde \pi^m$ defined with respect to $\wt P_m$, and denoted by $\optimisticctg{m}$, is finite. By \cref{lem:bertsekas-proper}, the following Bellman optimality equations hold for all $s \in S$,
\begin{align} \label{eq:hoff-optimisitc-bellman}
    \optimisticctg{m}(s) 
    & = 
    \min_{a \in A} c(s,a) + \sum_{s' \in S} \wt P_m(s' \mid s,a) \optimisticctg{m}(s').
\end{align}


%
%

\paragraph{High probability events.}
For every interval $m$, we let  $\geventi{m}$ denote the event that the confidence set for interval $m$ contains the true transition function $P$.
Formally, let $\bar{P}_m$ denote the empirical estimate of the transition function  at the beginning of interval $m$, let $N_m(s,a)$ denote the number of visits to state-action pair $(s,a)$ up to interval $m$ (not including), and let $n_m(s,a)$ be the number of visits to $(s,a)$ during interval $m$. Then we say that $\geventi{m}$ holds if for all $(s,a) \in S \times A$, we have ($N_+^m (s,a) = \max \{ 1 , N_m(s,a) \}$)
\begin{equation}
    \label{eq:hoff-conf-set}
    \lVert P (\cdot | s,a) - \Bar P_{m} (\cdot | s,a) \rVert_1
    \le
    \cshoff{m}{s}{a}.
\end{equation}

In the following lemma we show that, with high probability, the events $\geventi{m}$ hold and that the total cost in each interval is bounded.
Combining this with \cref{obs:num-intervals} we get that all episodes terminate within a finite number of steps, with high probability.

\begin{lemma}
\label{lem:hoff-num-time-steps}
With probability at least $1-\delta/2$, for all intervals $m$ simultaneously, we have that $\geventi{m}$ holds and that $\sum_{h=1}^{\Hm} c(s_h^m,a_h^m) \le \highprobcostbound{m}$, where $\Hm$ denotes the length of interval $m$, $s_h^m$ is the observed state at time $h$ of interval $m$ and $a_h^m  = \tilde{\pi}^m(s_h^m)$ is the chosen action.
This implies that the total number of steps of the algorithm is
\begin{align*}
    \totaltime 
    = 
    O \biggl( \frac{K \costbound}{\cmin} \logfactor + \frac{\costbound^3 |S|^2 |A|}{\cmin^3} \logfactor^2 \biggr).
\end{align*}
\end{lemma}

\begin{proof}[Proof sketch.]
 The events $\geventi{m}$ hold with high probability due to standard concentration inequalities, 
 and thus it remains to address the high probability bound on the total cost within each interval. 
    
    This proof consists of three parts. 
    In the first, we show that when $\geventi{m}$ occurs we have that $\optimisticctg{m}(s) \le \ctgopt(s) \le \costbound$ for all $s \in S$ due to the optimistic nature of the computation of $\tilde \pi^m$.
    In the second part, we postulate that had all state-action pairs been known, then having $\geventi{m}$ hold implies that $\ctg{m}(s) \le 2 \costbound$ for all $s \in S$. That is, when all state-action pairs are known, not only $\tilde \pi^m$ is proper in the true model, but its expected cumulative cost is at most $2 \costbound$.
    
    The third part of the proof deals with the general case when  not all state-action pairs are known.
    Fix some interval $m$. 
    Since the interval ends when we reach an unknown state-action, it must be that all but the first state-action pair visited during the interval are known.
    For this unknown first state-action pair, it follows from the Bellman equations (\cref{eq:hoff-optimisitc-bellman}) and from $\optimisticctg{m}(s) \le \costbound$ for all $s \in S$ that $\tilde \pi^m$ never picks an action whose instantaneous cost is larger than $\costbound$. 
    Therefore, the cost of this first unknown state-action pair is at most $\costbound$, and we focus on bounding the total cost in the remaining time steps with high probability.

    To that end, we define the following modified MDP $M^\text{know} = (S^\text{know}, A, P^\text{know}, c, \sinit)$ in which every state $s \in S$ such that $(s,\tilde{\pi}^m(s))$ is unknown is contracted to the goal state.
    Let $P^\text{know}$ be the transition function induced in $M^\text{know}$ by $P$, and let $\ctg{m}_\text{know}$ be the cost-to-go of $\tilde{\pi}^m$ in $M^\text{know}$ w.r.t $P^\text{know}$.
    Similarly, define $\wt P^\text{know}_m$ as the transition function induced in $M^\text{know}$ by $\wt P_m$, and $\optimisticctg{m}_\text{know}$ as the cost-to-go of $\tilde{\pi}^m$ in $M^\text{know}$ w.r.t $\wt P^\text{know}_m$.
    It is clear that $\optimisticctg{m}_\text{know} (s) \leq \optimisticctg{m} (s)$ for every $s \in S$ from whence  $\optimisticctg{m}_\text{know} (s) \leq \costbound$.
    Moreover, since all states $s \in S$ for which $(s,\tilde{\pi}^m(s))$ is unknown were contracted to the goal state, in $M^\text{know}$ all remaining states-action pairs are known. Therefore, by the second part of the proof, $\ctg{m}_\text{know}(s) \le 2\costbound$ for all $s \in S$.
    Note that reaching the goal state in $M^\text{know}$ is equivalent to reaching either the goal state or an unknown state-action pair in the true model hence the latter argument shows that the total expected cost in doing so is at most $2 \costbound$.
    We further obtain the high probability bound by a probabilistic amplification argument using the Markov property of the MDP.
\end{proof}

\paragraph{Regret analysis.}

In what follows, instead of bounding $R_K$, we bound
$\tregret{K} = \sum_{m=1}^\numintervals \sum_{h=1}^{\Hm} c(s_h^m, a_h^m) \indgeventi{m} - K \cdot \ctgopt(\sinit)$, where $\ind$ is the indicator function.
Note that according to \cref{lem:hoff-num-time-steps}, we have that $\tregret{K} = \regret$ with high probability. 

The definition of $\tregret{K}$ allows the analysis to disentangle two dependent probabilistic events. The first is the intersection of the events $\geventi{m}$ which is dealt with in \cref{lem:hoff-num-time-steps}. The second holds when, for a fixed policy, the costs suffered by the learner do not deviate significantly from their expectation. 
In the following lemma we bound $\tregret{K}$.
\begin{lemma} \label{lem:hoff-tilde-regret}
    With probability at least $1 - \delta/2$, we have
    \begin{align*}
        \tregret{K}
        \le
        O \biggl(\underbrace{
        \frac{\costbound^3 |S|^2 |A|}{\cmin^2} \log \frac{\costbound |S| |A|}{\cmin \delta}}_{(1)}
        +
        \costbound \sqrt{\totaltime \log \frac{\totaltime}{\delta}}
        +
        \costbound \sqrt{|S| \log \frac{|S| |A| \totaltime}{\delta}} \underbrace{\sum_{s,a} \sum_{m=1}^\numintervals \frac{n_m(s,a)}{\sqrt{N_+^{m}(s,a)}}}_{(2)}\biggr).
    \end{align*}
\end{lemma}
Here we only explain how to interpret the resulting bound.
The term (1) bounds the total cost spent in intervals that ended in unknown state-action pairs (it does not depend on $K$).
The term $(2)$ is at most 
$
  O(\sqrt{|S| |A| \totaltime})
$
when \cref{lem:hoff-num-time-steps} holds,
and then the dominant term in \cref{lem:hoff-tilde-regret} becomes $\tO(B |S| \sqrt{|A| T})$.
\cref{thm:reg-bound-knownb} is finally obtained by applying a union bound on \cref{lem:hoff-num-time-steps,lem:hoff-tilde-regret} and using \cref{lem:hoff-num-time-steps} to bound $\totaltime$. 

\subsection{Unknown Cost Bound} 
\label{sec:unkownb}

In this section we relax the assumption that $\costbound$ is known to the learner.
Instead, we keep an estimate $\costboundestimate$ that is initialized to $\cmin$ and doubles every time the cost in interval $m$ (denoted as $C_m$) reaches $\highprobcostboundestimate{m}$.
By \cref{lem:hoff-num-time-steps}, with high probability, $\costboundestimate \leq 2 \costbound$.
We end an interval as before (once the goal state is reached or an unknown state-action pair is reached), but also when $\costboundestimate$ is doubled.
The algorithm for this case is presented in \cref{sec:appendix-alg} (\cref{alg:unkownb}).
Since $\costboundestimate$ changes, every state-action pair can become known once for every different value of $\costboundestimate$.

\begin{observation}
\label{obs:num-intervals-unknownb}
When $\costbound$ is unknown to the learner, the number of times a state-action pair can become known is at most $\log_2(\costbound / \cmin)$. The number of intervals $\numintervals$ is 
\[
    O \biggl( K + \frac{\costbound^2 |S|^2 |A|}{\cmin^2} \log^2 \frac{ \costbound |S| |A|}{\delta \cmin} \biggr).
\]
\end{observation}

\begin{lemma}
\label{lem:hoff-num-steps-unknownb}
When $\costbound$ is unknown, with probability at least $1-\delta/2$, for all intervals $m$ simultaneously, we have that $\geventi{m}$ holds and that $\sum_{h=1}^{\Hm} c(s_h^m,a_h^m) \le \highprobcostbound{m}$.
This implies that the total number of steps of the algorithm is
\[
    \totaltime 
    = 
    O \biggl( \frac{K \costbound}{\cmin} \logfactor + \frac{\costbound^3 |S|^2 |A|}{\cmin^3} \logfactor^3 \biggr).
\]
\end{lemma}

The analysis follows that of \cref{alg:knownb}. In particular, \cref{lem:hoff-tilde-regret} still holds (with $2 \costbound$ instead of $\costbound$), and jointly with \cref{lem:hoff-num-steps-unknownb} imply \cref{thm:reg-bound-unknownb}.

\section{Bernstein-type Confidence Bounds} \label{sec:bernstein}

\begin{algorithm}[t]
    \caption{\sc Bernstein-type confidence bounds}
    \label{alg:bern-o-ssp}
    \begin{algorithmic} 
        \STATE {\bfseries input:} state space $S$, action space $A$ and confidence parameter $\delta$.
        
        \STATE {\bfseries initialization:} $i \gets 1$, $t \gets 1$, arbitrary policy $\tilde \pi_1$ , 
        $\forall (s,a,s') \in S \times A \times S: \; N_1(s,a,s') \gets 0 , N_1(s,a) \gets 0$, $n_1(s,a,s') \gets 0$, $n_1(s,a) \gets 0$.
        
        \FOR{$k=1,2,\ldots$}
            
            \STATE {\bf set} $s_t \gets \sinit$. 
                       
            \WHILE{$s_t \neq \ssink$}
                \STATE follow optimistic optimal policy: $a_t \gets \Tilde{\pi}_i(s_t)$.
                \STATE observe next state $s_{t+1} \sim P(\cdot \mid s_t,a_t)$.
                \STATE {\bf set}:
                $n_i(s_t, a_t) \gets n_i(s_t, a_t) + 1$, $n_i(s_t, a_t, s_{t+1}) \gets n_i(s_t, a_t, s_{t+1}) + 1$.
                \IF{$n_i(s_{t+1}, \tilde{\pi}_i(s_{t+1})) < N_i(s_{t+1},\tilde{\pi}_i(s_{t+1}))$}
                    \STATE {\bf set} $t \gets t + 1$ and {\bf continue}.
                \ENDIF
        
            \STATE \# start new epoch
            \STATE {\bf set}:
                $N_{i+1}(s,a,s') \gets N_i(s,a,s') + n_i(s,a,s')$,
                $N_{i+1}(s,a) \gets N_i(s,a) + n_i(s,a)$, 
                $n_{i+1}(s,a) \gets 0$,
                $n_{i+1}(s,a,s') \gets 0$
                for all $(s,a,s') \in S \times A \times S$.
                
                \STATE {\bf compute} empirical transition function $\Bar{P}$ as $\Bar{P} (s' \mid s,a) = N(s,a,s') / N_+(s,a)$ for every $(s,a,s') \in S \times A \times S$ where $N_+(s,a) = \max\{N(s,a),1\}$.
                
                \STATE {\bf compute} optimistic transition function $\wt P$ using \cref{eq:bern-optimistic-transition}.
                
                \STATE {\bf compute} optimal policy $\Tilde{\pi}$ w.r.t $\wt P$.
  
                \STATE $i \gets i + 1$, $t \gets t + 1$.
            \ENDWHILE
        \ENDFOR
    \end{algorithmic}
\end{algorithm}

\cref{alg:knownb} has two drawbacks. 
The first one is the use of Hoeffding-style confidence bounds which we improve with Bernstein-style confidence bounds.
The second is the number of times the optimistic optimal policy is computed.
In this section we propose to compute it in a way similar to UCRL2, i.e., once the number of visits to some state-action pair is doubled.
Note that this change also eliminates the need to know or to estimate $\costbound$.

The algorithm is presented in \cref{alg:bern-o-ssp}.
It consists of {\em epochs}.
The first epoch starts at the first time step, and each epoch ends once the number of visits to some state-action pair is doubled.
An optimistic policy is computed at the end of every epoch using (empirical) Bernstein confidence bounds.
In contrast to \cref{alg:knownb},
\cref{alg:bern-o-ssp} 
defines a confidence range for
each state, action, and next state, separately, around its empirical estimate (i.e., we use an $L_\infty$ ``ball'' rather than an $L_1$ ``ball'' around the empirical estimates).
This allows us to disentangle the computation of the optimistic policy from the computation of the optimistic model. 
Indeed, the computation of the optimistic model becomes very easy: one simply has to maximize the probability of transition directly to the goal state at every state-action pair which means minimizing the probability of transition to all other states and setting them at the lowest possible value of their confidence range. This results in the following formula for $\wt P (s' | s,a)$: 
\begin{align} 
    \max \{\bar P (s' | s,a ) - 28A(s,a) 
    -
    4 \sqrt{\bar{P}(s' | s,a) A(s,a)} , 0 \},
    \label{eq:bern-optimistic-transition}
\end{align}
where $A(s,a) = \log (|S||A| N_{+}(s,a) / \delta) / N_{+}(s,a)$.
The optimistic policy is then the optimal policy in the SSP model defined by the transition function $\wt P$.

\subsection{Analysis} \label{sec:bern-analysis}

In this section we prove \cref{thm:bern-reg-bound}.
We start by showing that our new confidence sets contain $P$ with high probability which implies that each episode ends in finite time with high probability.
Consequently, we are able to bound the regret through summation of our confidence bounds.

We once again distinguish between \emph{known} and \emph{unknown} state-action pairs similarly to \cref{alg:knownb}. 
A state-action pair $(s,a)$ becomes \emph{known} at the end of an epoch if the total number of visits to $(s,a)$ has passed $\alpha \cdot \numvisitsuntilknownbern$ at some time step during the epoch (for some constant $\alpha > 0$).
Note that at the end of the epoch, the visit count of $(s,a)$ may be strictly larger than $\alpha \cdot \numvisitsuntilknownbern$ but at most twice as much by the definition of our algorithm. 
Furthermore, we split each epoch into {\em intervals} similar to what did in \cref{sec:algorithm}.
The first interval starts at the first time step and each interval ends once (1) the total cost in the interval accumulates to at least $\costbound$; (2) an unknown state-action pair is reached; (3) the current episode ends; or (4) the current epoch ends. We have the following observation.

\begin{observation}\label{obs:cost-bounds-intervals}
    Let $C_M$ denote the cost of the learner after $M$ intervals.
    Observe that the total cost in each interval is at least $\costbound$ unless the interval ends in the goal state, in an unknown state-action pair or the epoch ends. Thus the total number of intervals satisfies
    \[
        M
        \le
        \frac{C_M}{\costbound} + 2 |S| |A| \log T + K + O \biggl(\frac{\costbound |S|^2 |A|}{\cmin} \log \frac{\costbound |S| |A|}{\delta \cmin} \biggr),
    \]
    and the total time satisfies
    $
        T \le C_M / \cmin.
    $
\end{observation}


Recall that in the analysis of \cref{alg:knownb} we show that once all state-action pairs are known, the optimistic policies generated by the algorithm are proper in the true MDP. The same holds true for \cref{alg:bern-o-ssp}, yet we never prove this directly.
Instead, our proof goes as follows.\footnote{We neglect low order terms here.} We prove that $C_M$, the cost accumulated by the learner during the first $M$ intervals, is at most  $K \cdot \ctgopt(\sinit) + \costbound \sqrt{M}$ with high probability as long as no more than $K$ episodes have been completed during these $M$ intervals. We notice that once all state-action pairs are known, the total cost in each interval is at least $\costbound$  (ignoring intervals that end with the end of an epoch or an episode), which implies that the total number of intervals $M$ is bounded by  $C_M/\costbound$. This allows us to  get a bound on $C_M$ that is independent of the number of intervals by solving the inequality $C_M \lesssim K \cdot \ctgopt(\sinit) + \costbound \sqrt{M} \lesssim K \cdot \ctgopt(\sinit) + \sqrt{\costbound \cdot C_M}$. From this, and since the instantaneous costs are strictly positive (by \cref{ass:c-min}), it must be that the learner eventually completes all $K$ episodes; i.e., there must be a time from which \cref{alg:bern-o-ssp} generates only proper policies.

\vspace{-1em}
\paragraph{Notation.}
The epoch that interval $m$ belongs to is denoted by $i(m)$, other notations are as in \cref{sec:analysis}.
Note that since the optimistic policy is computed at the end of an epoch and not at the end of an interval, it follows that $\tilde{\pi}^m = \tilde{\pi}^{i(m)}$ and $\optimisticctg{m} = \optimisticctg{i(m)}$.
The trajectory visited in interval $m$ is denoted by $\traj{m} = ( s_1^m, a_1^m , \ldots , s_{\Hm}^m, a_{\Hm}^m, s_{\Hm+1}^m)$, where $a_h^m$ is the action taken in $s_h^m$, and $\Hm$ is the length of the interval.
In addition, the concatenation of the trajectories of the intervals up to and including interval $m$ is denoted by $\trajconcat{m}$, that is 
$\trajconcat{m} = \cup_{m'=1}^m \traj{m'}$.

\paragraph{High probability events.}

Throughout the analysis we denote $S^+ = S \cup \{\ssink\}$.
For every interval $m$ we let $\geventi{m}$ denote the event that the confidence set for epoch $i = i(m)$ contains the actual transition function $P$.
Formally, if $\geventi{m}$ holds then for all $(s,a,s') \in S \times A \times S^+$, we have (denote $N_+^{m}(s,a) = \max\{1, N^{m}(s,a)\}$, 
$
    A_h^m
    = 
    A(s_{h}^m,a_{h}^m)
$)
\begin{align}
    | P (s' | s,a) - \Bar P_{m} (s' | s,a) |
    \le  
    28 A_h^m
    +
    4 \sqrt{\bar{P}_{m}(s' | s,a) A_h^m}.
    \label{eq:bernsteinconcentration}
\end{align}

In the following lemma 
we show that the events $\geventi{m}$ hold with high probability.
\begin{lemma}
\label{lem:bern-num-time-steps}
With probability at least $1-\delta/2$, $\Omega^m$ holds for all intervals $m$ simultaneously.
\end{lemma}

\paragraph{Regret analysis.}

In the following section, instead of bounding $\regret$, we bound $\tregret{M} = \sum_{m=1}^\numintervals \sum_{h=1}^{\Hm} c(s_h^m, a_h^m) \indgeventi{m} - K \ctgopt(\sinit)$ for any number of intervals $M$. 
This implies \cref{thm:bern-reg-bound} by the following argument.
\cref{lem:bern-num-time-steps} implies that $\tregret{M} = R_M$ with high probability for any number of intervals $M$ ($R_M$ is the true regret within the first $M$ intervals). In particular, when $M$ is the number of intervals in which the first $K$ episodes elapse, this implies \cref{thm:bern-reg-bound} (we show that the learner indeed completes these $K$ episodes).

To bound $\tregret{M}$, we use the next lemma to decompose $\tregret{M}$ into two terms which we bound independently.

\begin{lemma}
\label{lem:bern-reg-decomp}
It holds that
$
    \tregret{M} 
    =
    \sum_{m=1}^\numintervals \tregret{m}^1
    +
    \sum_{m=1}^\numintervals \tregret{m}^2
     - 
     K \cdot \ctgopt(\sinit),
$
where
\begin{alignat*}{2}
    &\tregret{m}^1 &&= \bigl( \optimisticctg{m}(s_1^m) - \optimisticctg{m}(s_{\Hm+1}^m) \bigr) \indgeventi{m}, 
    \quad \text{and} \\
    &\tregret{m}^2 &&= \Biggl( \sum_{h=1}^{\Hm} \optimisticctg{m}(s_{h+1}^m) - \sum_{s' \in S} \wt P_{m}(s' \mid s_h^m,a_h^m) \optimisticctg{m}(s') \Biggr) \indgeventi{m}.
\end{alignat*}
\end{lemma}

The lemma breaks down $\tregret{M}$ into two terms.
The first term accounts for the number of times in which the learner changes her policy in the middle of an episode which is at most the number of epochs.
The second term sums the errors between the cost-to-go of the observed next state and its estimated expectation. 

Indeed, $\sum_{m=1}^M \tregret{m}^1$ is related to the total number of epochs which is at most $|S| |A| \log_2 T$ due to the following lemma.

%

\begin{lemma}
\label{lem:bern-first-reg-term-bound}
    It holds that
    $
        \sum_{m=1}^\numintervals \tregret{m}^1 \le 2 \costbound |S| |A| \log \totaltime + K \ctgopt(\sinit).
    $
\end{lemma}

The next lemma shows that $\sum_{m=1}^M \tregret{m}^2$ does not deviate from $\sum_{m=1}^\numintervals \bbE[ \tregret{m}^2 \mid \trajconcat{m-1}]$ significantly.

\begin{lemma}
    \label{lem:optimistic-val-func-diff-to-expected}
    With probability at least $1 - \delta / 4$, 
    \begin{align}
        \sum_{m=1}^\numintervals \tregret{m}^2
        \le
        \sum_{m=1}^\numintervals \bbE \bigl[ \tregret{m}^2 \mid \trajconcat{m-1} \bigr]
        +
        3 \costbound \sqrt{\numintervals \log \frac{8 \numintervals}{\delta}}.
        \nonumber
    \end{align}
\end{lemma}

The key property of the lemma is that the deviations between $\sum_{m=1}^\numintervals \tregret{m}^2$ and its corresponding expectation is of order $\sqrt{M}$ and do not scale with $T$. 

To prove the lemma, we recall that an interval ends at most at the first time step in which the accumulated cost in the interval surpasses $\costbound$. We show in our analysis that $\optimisticctg{m}(s) \le \ctgopt(s) \le \costbound$ for all $s \in S$ due to the optimistic  computation of $\tilde \pi^m$. Therefore, $\tilde \pi^m$ never picks an action whose instantaneous cost is more than $\costbound$. This implies that the total cost within each interval is at most $2 \costbound$. 
Then, we use the Bellman equations to bound $\tregret{m}^2$ by order of the total cost in the interval, 
and the lemma follows by an application of Azuma's concentration inequality.

\cref{lem:sum-bern-bounds} below bounds $\bbE \bigl[ \tregret{m}^2 \mid \trajconcat{m-1} \bigr]$ for every interval $m$ by a sum of the confidence bounds used in \cref{alg:bern-o-ssp}. 
\begin{lemma}
    \label{lem:sum-bern-bounds}
    For every interval $m$, 
    \begin{align}
        \bbE \bigl[ \tregret{m}^2 \mid \trajconcat{m-1} \bigr] 
        \le 
        16 \bbE \Biggl[ 
        \sum_{h=1}^{\Hm} \sqrt{|S| \bbV_h^m A_h^m} \indgeventi{m} 
        \biggm| \trajconcat{m-1} \Biggr]
        + 
        272 \bbE \Biggl[ \sum_{h=1}^{\Hm} \costbound |S| A_h^m \indgeventi{m} 
        \biggm| \trajconcat{m-1} \Biggr],
         \label{eq:bern-cb-bound}
    \end{align}
    where $\bbV_h^m$ is the empirical variance defined as
    $
        \bbV_h^m 
        = 
        \sum_{s' \in S^+} P(s' \mid s_{h}^m,a_{h}^m) \bigl(\optimisticctg{m}(s') - 
        \mu_h^m \bigr)^2,
    $
    and $\mu_h^m = \sum_{s' \in S^+} P(s' \mid s_{h}^m,a_{h}^m) \optimisticctg{m}(s')$.
\end{lemma}

The next step is the part of our proof in which our analysis departs from that of \cref{alg:knownb}. 
Note that when $\geventi{m}$ holds, $\bbV_h^m \le \costbound^2$.
Using this bound for each time step separately will result in a bound similar to that of \cref{thm:reg-bound-unknownb}. 
However, this bound is loose due to the following intuitive argument. Suppose that we replace $\optimisticctg{m}$ with the true cost-to-go function of $\tilde \pi^{m}$, $\ctg{m}$, in the definition of $\bbV_h^m$. Note that from the Bellman equations (\cref{eq:bellman}) we have $\ctg{m}(s_h^m) > \ctg{m}(s_{h+1}^m)$ in expectation on consecutive time steps $h$ and $h+1$ hence we surmise that in expectation $\bbV_h^m$ would also decrease on consecutive time steps.
A similar argument holds when in reality we use $\optimisticctg{m}$ because all-but-one of the state-action pairs in the interval are known, and $\optimisticctg{m}$ is a ``close enough'' approximation of $\ctg{m}$ on known state-action pairs since they have been sampled sufficiently many times.
Indeed, in \cref{lem:bounded-V-h-i} we use the technique of \citet{azar2017minimax} to show that (up to a constant) $\costbound^2$ bounds the expected sum of the variances over the time steps of an interval. 

\begin{lemma}
\label{lem:bounded-V-h-i}
    $
        \bbE \bigl[\sum_{h=1}^{\Hm} \bbV_h^m \indgeventi{m} \mid \trajconcat{m-1} \bigr]
        \le
        44 \costbound^2.
    $
\end{lemma}

Armed with \cref{lem:bounded-V-h-i}, we upper bound $\sum_{m=1}^\numintervals \bbE \bigl[ \tregret{m}^2 \mid \trajconcat{m-1} \bigr]$ by applying some algebraic manipulation on \cref{eq:bern-cb-bound}, and summing over all intervals which gives the next lemma.

\begin{lemma}
\label{lem:sum-bern-bounds-cont}
With probability at least $1-\delta/4$, 
\begin{align*}
    \sum_{m=1}^\numintervals \bbE \bigl[ \tregret{m}^2 \mid \trajconcat{m-1} \bigr]
    \le
    614 \costbound \sqrt{\numintervals |S|^2 |A| \log^2 \frac{\totaltime |S| |A|}{\delta}}
    + 
    8160 \costbound |S|^2 |A| \log^2 \frac{\totaltime |S| |A|}{\delta}.
\end{align*}
\end{lemma}


\cref{thm:bern-reg-bound} is obtained by first applying a union bound on \cref{lem:bern-num-time-steps,lem:sum-bern-bounds-cont,lem:optimistic-val-func-diff-to-expected}, plugging in the bounds of \cref{lem:bern-first-reg-term-bound,lem:optimistic-val-func-diff-to-expected,lem:sum-bern-bounds-cont} into \cref{lem:bern-reg-decomp}, and bounding $\totaltime$ and $\numintervals$ using \cref{obs:cost-bounds-intervals}. This results in a quadratic inequality in $\sqrt{C_M}$ and solving it yields the theorem. 

\subsection*{Acknowledgements}

HK is supported by the Israeli Science Foundation (ISF) grant 1595/19.

%% file: appendix.tex
\appendix

\section{Algorithm}
\label{sec:appendix-alg}

\begin{algorithm}[H]
    \caption{\sc Hoeffding-type confidence bounds}
    \label{alg:unkownb}
    \begin{algorithmic}
    
    \STATE {\bfseries input:} state space $S$, action space $A$ and confidence parameter $\delta$.

     \STATE {\bfseries initialization:} arbitrary policy $\tilde{\pi}$, $m \gets 1, \costboundestimate \gets \cmin, C_1 \gets 0,
     \forall (s,a,s') \in S \times A \times S: \quad N(s,a,s') \gets 0 , N(s,a) \gets 0$.

     \FOR{$k=1,2,\ldots$}
     
     \STATE set $s \gets \sinit$.
    
    \WHILE{$s \neq \ssink$}

    \STATE follow optimistic optimal policy:
     $a \gets \Tilde{\pi}(s)$.
     
     \STATE suffer cost: $C_m \gets C_m + c(s,a)$.
     
     \STATE observe next state $s' \sim P(\cdot \mid s,a)$.
     
     \STATE update visit counters:
     $
     N(s,a,s') \gets N(s,a,s')+1 , N(s,a) \gets N(s,a)+1
     $.

     \IF{$N(s',\tilde{\pi}(s')) \leq \numvisitsuntilknownhoffest$ or $s' = \ssink$ or $C_m \geq \highprobcostboundestimate{m}$}
     
     \STATE \# start new interval
     
     \IF{$C_m \geq \highprobcostboundestimate{m}$}
     
     \STATE update $\costbound$ estimate: $\costboundestimate \gets 2 \costboundestimate$.
     
     \ENDIF
     
     \STATE advance intervals counter: $m \gets m+1$.
     
     \STATE initialize cost suffered in interval: $C_m \gets 0$.
     
     \STATE {\bf compute} empirical transition function $\Bar{P}$ for every $(s,a,s') \in S \times A \times S:$ $$\Bar{P} (s' \mid s,a) = \frac{N(s,a,s')}{\max\{N(s,a), 1\} }.$$
     
     \STATE {\bf compute} policy $\Tilde{\pi}$ that minimizes the expected cost with respect to a transition function $\wt P$, such that for every $(s,a) \in S \times A$:
    \[
    	\bigl\lVert \wt P (\cdot \mid s,a) - \Bar{P} (\cdot \mid s,a) \bigr\rVert_1 
	    \leq 
	    \cshoff{}{s}{a}.
    \]
     
     \ENDIF
     
     \STATE set $s \gets s'$.
     
     \ENDWHILE
     
     \ENDFOR
    
    \end{algorithmic}
\end{algorithm}

\section{Proofs}

\subsection{Proofs for \cref{sec:analysis}}
\label{sec:hoff-proofs}

\subsubsection{Proof of \cref{lem:hoff-num-time-steps}} \label{sec:proof-hoff-num-time-steps}

\begin{lemma*}[restatement of \cref{lem:hoff-num-time-steps}]
With probability at least $1-\delta/2$, $\geventi{m}$ holds and $\sum_{h=1}^{\Hm} c(s_h^m,a_h^m) \le \highprobcostbound{m}$ for all intervals $m$ simultaneously.
This implies that the total number of steps of the algorithm is
\[
    \totaltime 
    = 
    O \biggl( \frac{K \costbound}{\cmin} \log \frac{K \costbound |S| |A|}{\delta \cmin } + \frac{\costbound^3 |S|^2 |A|}{\cmin^3} \log^2 \frac{K \costbound |S| |A|}{\delta \cmin } \biggr).
\]
\end{lemma*}

\begin{lemma}
\label{lem:hoff-CS-hold}
The event $\geventi{m}$ holds for all intervals $m$ simultaneously with probability at least $1 - \delta/4$.
\end{lemma}

\begin{proof}
Fix a state $s$ and an action $a$. 
Consider an infinite sequence $\{Z_i\}_{i=1}^\infty$ of draws from the distribution
$P(\cdot \mid s,a)$. 
By \cref{thm:weissman} we get that for a prefix of length $t$ of this sequence (that is $\{ Z_i \}_{i=1}^t$)
\[
    \bigl\lVert P (\cdot \mid s,a) - \Bar{P}_{\{Z_i\}_{i=1}^t} (\cdot \mid s,a) \bigr\rVert_1 
    \le
    2 \sqrt{\frac{ |S| \log (\delta_t^{-1})}{t}}\ ,
\]
holds with probability $1-\delta_t$,
where $\Bar{P}_{\{Z_i\}_{i=1}^t} (\cdot \mid s,a)$
is the empirical distribution defined by the draws $\{Z_i\}_{i=1}^t$. 
We repeat this argument for every prefix $\{Z_i\}_{i=1}^t$ of
$\{Z_i\}_{i=1}^\infty$
and for every state-action pair, with $\delta_t = \delta / 8 |S| |A| t^2$. Then from the union bound we get that $\geventi{m}$ holds for all intervals $m$ simultaneously with probability at least $1 - \delta/4$.
\end{proof}

\begin{lemma}
\label{lem:opt-val-bound}
Let $m$ be an interval. If $\geventi{m}$ holds then  $\optimisticctg{m}(s) \leq \ctgopt(s) \leq \costbound$ for every $s \in S$.
\end{lemma}

\begin{proof}
\citet{tarbouriech2019noregret} show that all the transition functions in the confidence set of \cref{eq:hoff-conf-set} can be combined into a single augmented MDP.
The optimal policy of the augmented MDP can be found efficiently, e.g., with Extended Value Iteration.
The optimistic policy is the optimal policy in the augmented MDP.
It minimizes $\optimisticctg{m} (s)$ over all policies and feasible transition functions, for all states $s \in S$ simultaneously (following \citealp{bertsekas1991analysis}).
Since $\geventi{m}$ holds, it follows that the real transition function is in the confidence set therefore it is also considered in the minimization.
Thus $\optimisticctg{m}(s) \le \ctgopt(s)$ for all $s \in S$.
Finally, $\ctgopt(s) \le \costbound$ by the definition of $\costbound$.
\end{proof}

\begin{lemma}
    \label{lem:known-state}
    Let $m$ be an interval and $(s,a)$ be a known state-action pair. If $\geventi{m}$ holds then 
    \[
        \lVert \wt P_m(\cdot \mid s,a) - P(\cdot \mid s,a) \rVert_1 \leq \frac{c(s,a)}{2 \costbound}\ .
    \]
\end{lemma}

\begin{proof}
By the definition of the confidence set
\[
    \lVert \wt P_m (\cdot \mid s,a) - \Bar{P}_m (\cdot \mid s,a) \rVert_1 
    \le
    \cshoff{m}{s}{a}
    \le
    \frac{c(s,a)}{4 \costbound},
\]
where the last inequality follows because $\log (x) / x$ is decreasing, and $N_+^m(s,a) \ge \numvisitsuntilknownhoff$ since $(s,a)$ is known.
Similarly, since $\geventi{m}$ holds we also have that
\[
\lVert P (\cdot \mid s,a) - \Bar{P}_m (\cdot \mid s,a) \rVert_1 
\le
\cshoff{m}{s}{a}
\le 
\frac{c(s,a)}{4 \costbound},
\]
and the lemma follows by the triangle inequality.
\end{proof}

\begin{lemma}
    \label{lem:close-mdps-proper}
    Let $\Tilde{\pi}$ be a policy and $\wt P$ be a transition function.
    Denote the cost-to-go of $\Tilde{\pi}$ with respect to $\wt P$ by $\optimisticctg{}$.
    Assume that for every $s \in S$, $\optimisticctg{} (s) \leq \costbound$ and that
    \[
        \bigl\lVert \wt P(\cdot \mid s,\Tilde{\pi}(s)) 
        - 
        P(\cdot \mid s,\Tilde{\pi}(s)) \bigr\rVert_1 
        \leq 
        \frac{c(s,\Tilde{\pi}(s))}{2 \costbound}.
    \]
    Then, $\Tilde{\pi}$ is proper (with respect to $P$), and it holds that $ \ctg{\tilde{\pi}} (s) \leq 2 \costbound$ for every $s \in S$.
\end{lemma}

\begin{proof}
Consider the Bellman  equations of $\Tilde{\pi}$ with respect to transition function $\wt P$ at some state $s \in S$ (see \cref{lem:bertsekas-proper}), defined as
\begin{align}
    \nonumber
    \optimisticctg{} (s) 
    & = 
    c(s,\Tilde{\pi}(s)) + \sum_{s' \in S} \wt P (s' \mid s,\tilde{\pi}(s)) \optimisticctg{}(s')
    \\
    \label{eq:opt-bell}
    & =
    c(s,\Tilde{\pi}(s)) + \sum_{s' \in S} P (s' \mid s,\tilde{\pi}(s)) \optimisticctg{}(s')
    +
    \sum_{s' \in S} \optimisticctg{}(s') \left( \wt P (s' \mid s,\tilde{\pi}(s)) - P(s' \mid s,\tilde{\pi}(s)) \right) \ .
\end{align}

Notice that by our assumptions and using H\"{o}lder inequality,
\begin{align*}
    \left| \sum_{s' \in S} \optimisticctg{}(s') \left( \wt P (s' \mid s,\tilde{\pi}(s)) - P(s' \mid s,\tilde{\pi}(s)) \right) \right| 
    & \leq 
    \lVert \wt P(\cdot \mid s,\Tilde{\pi}(s)) - P(\cdot \mid s,\Tilde{\pi}(s)) \rVert_1 \cdot \lVert \optimisticctg{} \rVert_\infty
    \\
    & \leq 
    \frac{c(s,\Tilde{\pi}(s))}{2 \costbound} \cdot \costbound 
    = 
    \frac{c(s,\Tilde{\pi}(s))}{2}\ .
\end{align*}

Plugging this into \cref{eq:opt-bell}, we obtain
\begin{align*}
    \optimisticctg{} (s) 
    \geq 
    c(s,\Tilde{\pi}(s)) + \sum_{s' \in S}  P (s' \mid s,\tilde{\pi}(s)) \optimisticctg{}(s')
    -
    \frac{c(s,\Tilde{\pi}(s))}{2} 
    = 
    \frac{c(s,\Tilde{\pi}(s))}{2} + \sum_{s' \in S}  P (s' \mid s,\tilde{\pi}(s)) \optimisticctg{}(s').
\end{align*}

Therefore, defining ${\ctg{}}' = 2\optimisticctg{}$,
then ${\ctg{}}'(s) \geq c(s,\Tilde{\pi}(s)) + \sum_{s' \in S} P (s' \mid s,\tilde{\pi}(s)) {\ctg{}}'(s')$ for all $s \in S$. 
The statement now follows by \cref{lem:bertsekas-proper}.
\end{proof}

\begin{lemma}
\label{lem:proper-not-run-log}
Let $\pi$ be a proper policy such that for some $v > 0$, $\ctg{\pi} (s) \leq v$ for every $s \in S$.  
Then, the probability that the cost of $\pi$ to reach the goal state from any state $s$ is more than $m$, is at most $2e^{-m/4v}$ for all $m \ge 0$.
Note that a cost of at most $m$ implies that the number of steps is at most $\tfrac{m}{\cmin}$.
\end{lemma}

\begin{proof}
    By Markov inequality, the probability that  $\pi$ accumulates cost of  more than $2v$ before reaching the goal state is at most $1/2$. 
    Iterating this argument, we get that the probability that $\pi$ accumulates cost of more than $2kv$ before reaching the goal state is at most $2^{-k}$ for every integer $k \geq 0$.
    In general, for any $m \ge 0$, the probability that $\pi$ suffers a cost of more than $m$ is at most
    $
        2^{-\lfloor m/2v \rfloor} \le 2 \cdot 2^{-m/2v} \le 2 e^{-m/4v}.
    $
\end{proof}

For the next lemma we will need the following definitions.
The trajectory visited in interval $m$ is denoted by $\traj{m} = ( s_1^m, a_1^m , \ldots , s_{\Hm}^m, a_{\Hm}^m, s_{\Hm+1}^m)$ where $a_h^m$ is the action taken in $s_h^m$, and $\Hm$ is the length of the interval.
In addition, the concatenation of the trajectories in the intervals up to and including interval $m$ is denoted by $\trajconcat{m} = \cup_{m'=1}^m \traj{m'}$.

\begin{lemma} \label{lem:hoff-interval-costbound}
    Let $m$ be an interval.
    For all $r \ge 0$, we have that
    \[
        \Pr \Biggl[\sum_{h=1}^{\Hm} c \bigl(s_h^m, a_h^m \bigr) \indgeventi{m} > r \mid \trajconcat{m-1} \Biggr] 
        \le 
        3e^{-r/8 \costbound}.
    \]
\end{lemma}

\begin{proof}
    Note that $\geventi{m}$ is determined given $\trajconcat{m-1}$, and suppose that $\geventi{m}$ holds otherwise $\sum_{h=1}^{\Hm} c \bigl(s_h^m, a_h^m \bigr) \indgeventi{m}$ is $0$.
    Also assume that $r \ge 8 \costbound$ or else the statement holds trivially.
    
    Define the MDP $M^\text{know} = (S^\text{know}, A, P^\text{know}, c, \sinit)$ in which every state $s \in S$ such that $(s,\tilde{\pi}^m(s))$ is unknown is contracted into  the goal state.
    Let $P^\text{know}$ be the transition function induced in $M^\text{know}$ by $P$, and let $\ctg{m}_\text{know}$ be the cost-to-go of $\tilde{\pi}^m$ in $M^\text{know}$ with respect to $P^\text{know}$.
    Similarly, define $\wt P^\text{know}_m$ as the transition function induced in $M^\text{know}$ by $\wt P_m$, and $\optimisticctg{m}_\text{know}$ as the cost-to-go of $\tilde{\pi}^m$ in $M^\text{know}$ with respect to $\wt P^\text{know}_m$.
    It is clear that $\optimisticctg{m}_\text{know} (s) \leq \optimisticctg{m}(s)$ for every $s \in S$, so by \cref{lem:opt-val-bound}, $\optimisticctg{m}_\text{know} (s) \leq \costbound$.
    Moreover, since all the states $s \in S$ for which $(s,\tilde{\pi}^m(s))$ is unknown were contracted to the goal state, we can use \cref{lem:known-state} to obtain for all $s \in S^\text{know}$:
    \begin{equation}
        \label{eq:known}
        \bigl\lVert \wt P^\text{know}_m(\cdot \mid s,\Tilde{\pi}^m(s)) 
        - 
        P^\text{know}(\cdot \mid s,\Tilde{\pi}^m(s)) \bigr\rVert_1
        \le
        \bigl\lVert \wt P_m(\cdot \mid s,\Tilde{\pi}^m(s)) 
        - 
        P(\cdot \mid s,\Tilde{\pi}^m(s)) \bigr\rVert_1 
        \le 
        \frac{c(s,\Tilde{\pi}^m(s))}{2 \costbound}.
    \end{equation}
    
    We can apply \cref{lem:close-mdps-proper} in $M^\text{know}$ and obtain that $\ctg{m}_{\text{know}} (s) \leq 2 \costbound$ for every $s \in S^\text{know}$. 
    Notice that reaching the goal state in $M^\text{know}$ is equivalent to reaching the goal state or an unknown state-action pair in $M$, and also recall that all state-action pairs in the interval are known except for the first one.
    Thus, from \cref{lem:proper-not-run-log},
    \begin{align*}
        \Pr \Biggl[\sum_{h=2}^{\Hm} c \bigl(s_h^m, a_h^m \bigr) \indgeventi{m} > r-\costbound \mid \trajconcat{m-1} \Biggr]
        \le
        2e^{-(r-\costbound)/8\costbound}
        \le
        3e^{-r / 8 \costbound}.
    \end{align*}
    
    Since $\optimisticctg{m} \le \costbound$, our algorithm will never select an action whose instantaneous cost is larger than $\costbound$. 
    Since the first state-action in the interval might not be known, its cost is at most $\costbound$, and therefore
    \begin{align*}
        \Pr \Biggl[\sum_{h=1}^{\Hm} c \bigl(s_h^m, a_h^m \bigr) \indgeventi{m} > r \mid \trajconcat{m-1} \Biggr]
        \le
        \Pr \Biggl[\sum_{h=2}^{\Hm} c \bigl(s_h^m, a_h^m \bigr) \indgeventi{m} > r-\costbound \mid \trajconcat{m-1} \Biggr]
        \le
        3 e^{-r / 8 \costbound}. \qquad \qedhere
    \end{align*}
\end{proof}

\begin{proof}[Proof of \cref{lem:hoff-num-time-steps}]
From \cref{lem:hoff-interval-costbound}, with probability at least $1 - \delta / 16 m^2$,
$\sum_{h=1}^{\Hm} c \bigl(s_h^m, a_h^m \bigr) \le \highprobcostbound{m}$, and by the union bound this holds for all intervals $m$ simultaneously with probability 
 at least $1 - \delta/4$.
 By  \cref{lem:hoff-CS-hold},
 with probability $1 - \delta / 4$, $\Omega^m$ holds for all intervals $m$.
Combining these two facts again by a union bound, we get that both $\geventi{m}$ holds and the cost of  interval $m$ is at most $\highprobcostbound{m}$ simultaneously to all intervals $m$ with probability at least $1 - \delta / 2$.

If the cost of all intervals is bounded (and therefore so is the length of the interval), we can use the bound on the number of intervals in
 \cref{obs:num-intervals} to conclude that
\begin{align*}
    T 
    & =
    O \left( \frac{\costbound}{\cmin} \log \frac{\numintervals}{\delta} \cdot \left( K + \frac{\costbound^2 |S|^2 |A|}{\cmin^2} \log \frac{ \costbound |S| |A|}{\delta \cmin} \right) \right) 
    \\
    & =
    O \left( \frac{K \costbound}{\cmin} \log \frac{K \costbound |S| |A|}{\delta \cmin} + \frac{\costbound^3 |S|^2 |A|}{\cmin^3} \log^2 \frac{K \costbound |S| |A|}{\delta \cmin} \right).  \qedhere
\end{align*}
\end{proof}

\subsubsection{Proof of \cref{lem:hoff-tilde-regret}} \label{sec:proof-hoff-tilde-regret}

\begin{lemma*}[restatement of \cref{lem:hoff-tilde-regret}]
    With probability at least $1 - \delta/2$, we have
    \begin{align*}
        \tregret{K}
        \le
        \frac{5000 \costbound^3 |S|^2 |A|}{\cmin^2} \log \frac{\costbound |S| |A|}{\cmin \delta}
        +
        \costbound \sqrt{\totaltime \log \frac{4 \totaltime}{\delta}}
        +
        10 \costbound \sqrt{|S| \log \frac{|S| |A| \totaltime}{\delta}} \sum_{s,a} \sum_{m=1}^\numintervals \frac{n_m(s,a)}{\sqrt{N_+^{m}(s,a)}}.
    \end{align*}
\end{lemma*}

To analyze $\tregret{K}$, we begin by plugging in the Bellman optimality equation of $\Tilde{\pi}^m$ with respect to $\wt P_m$ into $\tregret{K}$. This allows us to decompose $\tregret{K}$ into three terms as follows.
\begin{align}
    \tregret{K}
    &=
    \sum_{m=1}^\numintervals \sum_{h=1}^{\Hm} \left( \optimisticctg{m}(s_h^m) - \sum_{s' \in S} \wt P_m(s' \mid s_h^m,a_h^m) \optimisticctg{m}(s') \right) \indgeventi{m} 
    -
    K \cdot \ctgopt(\sinit)
    \nonumber
    \\
    &= \sum_{m=1}^\numintervals \sum_{h=1}^{\Hm} \left( \optimisticctg{m}(s_h^m) -  \optimisticctg{m}(s_{h+1}^m) \right) \indgeventi{m} 
    -
    K \cdot \ctgopt(\sinit)
    \label{eq:hoff-reg-decomp-1}
    \\
    \label{eq:hoff-reg-decomp-2}
    & \qquad +
    \sum_{m=1}^\numintervals  \sum_{h=1}^{\Hm}  \sum_{s' \in S}  \optimisticctg{m}(s') \left( P(s' \mid s_h^m,a_h^m) -  \wt P_m(s' \mid s_h^m,a_h^m) \right) \indgeventi{m}
    \\
    \label{eq:hoff-reg-decomp-3}
    & \qquad + 
    \sum_{m=1}^\numintervals \left( \sum_{h=1}^{\Hm} \optimisticctg{m}(s_{h+1}^m) - \sum_{s' \in S} P(s' \mid s_h^m,a_h^m) \optimisticctg{m}(s') \right) \indgeventi{m}.
\end{align}

\cref{eq:hoff-reg-decomp-1} is a bound on the cost suffered from switching policies each time we visit an unknown state-action pair and is bounded by the following lemma.

\begin{lemma} \label{lem:hoff-switching-cost}
    $
        \sum_{m=1}^M
        \sum_{h=1}^{\Hm} 
        \left( \optimisticctg{m}(s_h^m) -  \optimisticctg{m}(s_{h+1}^m) \right) \indgeventi{m} 
        \le
        \costbound |S| |A| \cdot \numvisitsuntilknownhoff + K \cdot \ctgopt(\sinit).
    $
\end{lemma}

\begin{proof}
Note that per interval $\sum_{h=1}^{\Hm} (\optimisticctg{m}(s_h^m) -  \optimisticctg{m}(s_{h+1}^m))$ is a telescopic sum which equals $\optimisticctg{m}(s_1^m) -  \optimisticctg{m}(s_{\Hm+1}^m)$.
Furthermore, for every two consecutive intervals $m,m+1$ one of the following occurs:
\begin{enumerate}[label=(\roman*)]
    \item If interval $m$ ended in the goal state then
    $
        \optimisticctg{m}(s_{\Hm + 1}^m) = \optimisticctg{m}(\ssink) = 0
    $
    and 
    $
        \optimisticctg{m+1}(s_1^{m+1}) = \optimisticctg{m+1}(\sinit).
    $
    Thus, using \cref{lem:opt-val-bound} for the last inequality,
    \[
    \optimisticctg{m+1}(s_1^{m+1}) \indgeventi{m+1}  - 
    \optimisticctg{m}(s_{\Hm + 1}^m) \indgeventi{m} 
    = 
    \optimisticctg{m+1}(\sinit) \indgeventi{m+1}
    \leq 
    \ctgopt(\sinit).
    \]
    This  happens at most $K$ times.
    
    \item If interval $m$ ended in an unknown state then
    \[
        \optimisticctg{m+1}(s_1^{m+1}) \indgeventi{m+1}  
        - 
        \optimisticctg{m}(s_{\Hm + 1}^m) \indgeventi{m} 
        \leq 
        \optimisticctg{m+1}(s_1^{m+1}) \indgeventi{m+1}
        \leq 
        \costbound.
    \]
    This happens at most $|S| |A| \cdot \numvisitsuntilknownhoff$ times. \qedhere
\end{enumerate}
\end{proof}

\cref{lem:sum-conf-bounds} bounds \cref{eq:hoff-reg-decomp-2}  using techniques borrowed from \citet{AuerUCRL}. 

\begin{lemma}
\label{lem:sum-conf-bounds}
It holds that
\[
    \sum_{m=1}^\numintervals  \sum_{h=1}^{\Hm}  \sum_{s' \in S}  \optimisticctg{m}(s') \left( P(s' \mid s_h^m,a_h^m) -  \wt P_m(s' \mid s_h^m,a_h^m) \right) \indgeventi{m}
    \le
    10 \costbound \sqrt{|S| \log \frac{|S| |A| \totaltime}{\delta}} \sum_{s,a} \sum_{m=1}^\numintervals \frac{n_m(s,a)}{\sqrt{N_+^{m}(s,a)}}.
\]
\end{lemma}

\begin{proof}
Using the definition of the confidence sets we obtain
\begin{align*}
    \sum_{m=1}^\numintervals  \sum_{h=1}^{\Hm}  \sum_{s' \in S}  \optimisticctg{m}(s') &  \left( P(s' \mid s_h^m,a_h^m) -  \wt P_m(s' \mid s_h^m,a_h^m) \right) \indgeventi{m}
    \le
    \\
    & \le
    \costbound \sum_{s \in S} \sum_{a \in A} \sum_{m=1}^\numintervals n_m(s,a)  \lVert P (\cdot \mid s,a) - \wt P_m (\cdot \mid s,a) \rVert_1 \indgeventi{m}
    \\
    & \le
    10 \costbound \sum_{s \in S} \sum_{a \in A} \sum_{m=1}^\numintervals n_m(s,a)  \sqrt{ \frac{ |S| \log \bigl(|S| |A| N_+^{m}(s,a) / \delta \bigr)}{N_+^{m}(s,a)}}
    \\
    & \le
    10 \costbound \sqrt{|S| \log \frac{|S| |A| \totaltime}{\delta}} \sum_{s \in S} \sum_{a \in A} \sum_{m=1}^\numintervals \frac{n_m(s,a)}{\sqrt{N_+^{m}(s,a)}}.
\end{align*}
where the first inequality follows from H\"{o}lder inequality and \cref{lem:opt-val-bound}, and the second because $\wt P_m$ and $P$ are both in the confidence set of \cref{eq:hoff-conf-set} when $\geventi{m}$ holds.
The third inequality follows because $N_+^m(s,a) \le \totaltime$.
\end{proof}

\cref{lem:exp-traj} bounds the term in \cref{eq:hoff-reg-decomp-3} using Azuma's concentration inequality.

\begin{lemma}
\label{lem:exp-traj}
With probability at least $1 - \delta / 2$,
\[
\sum_{m=1}^\numintervals \left( \sum_{h=1}^{\Hm} \optimisticctg{m}(s_{h+1}^m) - \sum_{s' \in S} P(s' \mid s_h^m,a_h^m) \optimisticctg{m}(s') \right) \indgeventi{m}
\leq 
\costbound \sqrt{\totaltime \log \frac{4 \totaltime}{\delta}}.
\]
\end{lemma}

\begin{proof}
Consider the infinite sequence of random variables
\[
    X_t
    = 
    \Biggl( \optimisticctg{m}(s_{h+1}^m) - \sum_{s' \in S} P(s' \mid s_h^m,\Tilde{\pi}^m(s_h^m)) \optimisticctg{m}(s') 
    \Biggr) \indgeventi{m},
\]
where $m$ is the interval containing time $t$, and $h$ is the index of time step $t$ within  interval $m$.
Notice that this is a martingale difference sequence, and $|X_t| \leq \costbound$ by \cref{lem:opt-val-bound}.
Now, we apply anytime Azuma's inequality (\cref{thm:azuma}) to any prefix of the sequence $\{ X_t \}_{t=1}^\infty$. Thus, with probability at least $1 - \delta / 2$, for every $T$:
\[
    \sum_{t=1}^\totaltime X_t
    \leq 
    \costbound \sqrt{\totaltime \log \frac{4 \totaltime}{\delta}}. \qedhere
\]
\end{proof}

\subsubsection{Proof of \cref{thm:reg-bound-knownb}} \label{sec:proof-reg-bound-knownb}

\begin{theorem*}[restatement of \cref{thm:reg-bound-knownb}]
    Suppose that \cref{ass:c-min} holds.
    With probability at least $1 - \delta$ the regret of \cref{alg:knownb} is bounded as follows:
    \begin{align*}
        \regret
        &=
        O \Biggl( \sqrt{ \frac{ \costbound^3 |S|^2 |A| K}{\cmin}} \log \frac{K \costbound |S| |A|}{\delta \cmin } + \frac{\costbound^3 |S|^2 |A|}{\cmin^2} \log^{3/2} \frac{K \costbound |S| |A|}{\delta \cmin } \Biggr).
    \end{align*}
\end{theorem*}

\begin{lemma}
\label{lem:num-visits-sum-bound}
Assume that the number of steps in every interval is
 is at most $\frac{24 \costbound}{\cmin} \log \frac{4m}{\delta}$.
Then
for every $s \in S$ and $a \in A$,
\[
\sum_{m=1}^\numintervals \frac{n_m(s,a)}{\sqrt{N_+^m(s,a)}}
\le
3 \sqrt{N_{\numintervals+1}(s,a)}\ .
\]
\end{lemma}

\begin{proof}
We claim that, by the assumption of the lemma, for every interval $m$ we have that $n_m(s,a) \le N_+^m(s,a)$.
Indeed, if $(s,a)$ is unknown  then $n_m(s,a) = 1$
and since $N_+^m(s,a)\ge 1$ the claim follows.
If $(s,a)$ is known then $N_+^m(s,a) \ge \numvisitsuntilknownhoff$ and by our assumption the length of the interval, and in particular $n_m(s,a)$, is at most $\frac{24 \costbound}{\cmin} \log \frac{4m}{\delta}$.
Our statement then follows by \citet[Lemma 19]{AuerUCRL}.
\end{proof}

\begin{proof}[Proof of \cref{thm:reg-bound-knownb}]
    With probability at least $1 - \delta$, both \cref{lem:hoff-num-time-steps,lem:exp-traj} hold. 
    \cref{lem:hoff-num-time-steps} states that the length of every interval is at most $\frac{24 \costbound}{\cmin} \log \frac{4m}{\delta}$, and \cref{lem:num-visits-sum-bound} obtains
    \begin{align}
    \label{eq:bound-inverse-visit-sum}
        \sum_{s \in S} \sum_{a \in A} \sum_{m=1}^\numintervals \frac{n_m(s,a)}{\sqrt{N_+^{m}(s,a)}} 
        \le
        3 \sum_{(s,a) \in S \times A} \sqrt{N_{\numintervals+1}(s,a)} 
        \le
        3 \sqrt{|S| |A| \totaltime},
    \end{align}
    where the last inequality follows from Jensen's inequality and the fact that  $\sum_{(s,a) \in S \times A} N_{\numintervals+1}(s,a) \le \totaltime$.
    Next, we sum the bounds of \cref{lem:hoff-switching-cost,lem:sum-conf-bounds,lem:exp-traj} and use \cref{eq:bound-inverse-visit-sum} to obtain
    \begin{align*}
        R_K
        & \le
        5000 \frac{\costbound^3 |S|^2 |A|}{\cmin^2} \log \frac{\costbound |S| |A|}{\delta \cmin}
        +
        30 \costbound |S| \sqrt{|A| \totaltime \log \frac{|S| |A| \totaltime}{\delta}} +
        \costbound \sqrt{\totaltime \log \frac{4 \totaltime}{\delta}}.
    \end{align*}
    To finish the proof use \cref{lem:hoff-num-time-steps} to bound $\totaltime$.
\end{proof}

\subsection{Proofs for \cref{sec:bern-analysis}}
\label{sec:bern-proofs}

\subsubsection{Proof of \cref{lem:bern-num-time-steps}}
\label{sec:proof-bern-num-time-steps}

\begin{lemma*}[restatement of \cref{lem:bern-num-time-steps}]
    With probability at least $1-\delta / 2$, $\geventi{m}$ holds for all intervals $m$ simultaneously.
\end{lemma*}

\begin{proof}
    Fix a triplet $(s,a,s') \in S \times A \times S^+$. 
    Consider an infinite sequence $(Z_i)_{i=1}^\infty$ of draws from the distribution $P(\cdot \mid s,a)$ and let
    $X_i = \ind \{ Z_i = s' \}$.
    We apply \cref{eq:anytime-bern-2} of \cref{thm:bernstein} with $\delta_t = \tfrac{\delta}{4 |S|^2 |A| t^2}$ to
    a prefix of length $t$ of the sequence $(X_i)_{i=1}^\infty$.
    Then divide \cref{eq:anytime-bern-2} by $t$ and obtain that, after simplifying using the assumptions that $|S| \ge 2$ and $|A| \ge 2$, \cref{eq:bernsteinconcentration} holds with probability $1-\delta_t$.
    We repeat this argument for every prefix $(Z_i)_{i=1}^t$ of $(Z_i)_{i=1}^\infty$ and for every state-action-state triplet. 
    Then from the union bound we get that $\geventi{m}$ holds for all intervals $m$ simultaneously with probability at least $1 - \delta / 2$. 
\end{proof}

\subsubsection{Proof of \cref{lem:bern-reg-decomp}}
\label{sec:proof-bern-reg-decomp}

\begin{lemma*}[restatement of \cref{lem:bern-reg-decomp}]
    It holds that
    \begin{align*}
        \tregret{M}
        &= 
        \sum_{m=1}^\numintervals \Biggl( \sum_{h=1}^{\Hm} \optimisticctg{m}(s_h^m) - \optimisticctg{m}(s_{h+1}^m) \Biggr) \indgeventi{m}
        - 
        K \cdot \ctgopt(\sinit) \\
        &\qquad +
        \sum_{m=1}^\numintervals \Biggl( \sum_{h=1}^{\Hm} \optimisticctg{m}(s_{h+1}^m) - \sum_{s' \in S} \wt P_{m}(s' \mid s_h^m,a_h^m) \optimisticctg{m}(s') \Biggr) \indgeventi{m}.
    \end{align*}
\end{lemma*}

\begin{lemma}
    \label{lem:bellman-optimistic}
    Let $m$ be an interval. If $\geventi{m}$ holds then
    $\tilde \pi^m$ satisfies the Bellman equations in the optimistic model:
    \[
        \optimisticctg{m}(s) 
        =
        c(s,\tilde \pi^m(s)) 
        +
        \sum_{s' \in S} \wt P_m(s' \mid s, \tilde \pi^i(s)) \optimisticctg{m}(s'),
        \quad \forall s \in S.
    \]
\end{lemma}

\begin{proof}
    Note that the Bellman equations hold in the optimistic model since as we defined this model, there is a nonzero probability of transition to the goal state by any action from every state. Thus in the optimistic model every policy is a proper policy and in particular \cref{lem:bertsekas-optimal} holds.
\end{proof}

\begin{proof}[Proof of \cref{lem:bern-reg-decomp}]
    By \cref{lem:bellman-optimistic}, we can use the Bellman equations in the optimistic model to have the following interpretation of the costs for every interval $m$ and time $h$:
    \begin{align}
        c(s_h^m, a_h^m) \indgeventi{m}
        &= 
        \Biggl( \optimisticctg{m}(s_h^m) - \sum_{s' \in S} \wt P_i(s' \mid s_h^m,a_h^m) \optimisticctg{m}(s') \Biggr) \indgeventi{m} \nonumber \\
        &= 
        \Biggl( \optimisticctg{m}(s_h^m) -  \optimisticctg{m}(s_{h+1}^m) \Biggr) \indgeventi{m}
        + 
        \Biggl( \optimisticctg{m}(s_{h+1}^m) - \sum_{s' \in S} \wt P_i(s' \mid s_h^m,a_h^m) \optimisticctg{m}(s') \Biggr) \indgeventi{m}. \label{eq:cost-formula}
    \end{align}
    We now write 
    $
        \tregret{M}
        =
        \sum_{m=1}^\numintervals \sum_{h=1}^{\Hm} c(s_h^m, a_h^m) \indgeventi{m} 
        - 
        K \cdot \ctgopt(\sinit),
    $ 
    and substitute for each cost using \cref{eq:cost-formula} to get the lemma.
\end{proof}

\subsubsection{Proof of \cref{lem:bern-first-reg-term-bound}}
\label{sec:proof-bern-first-reg-term-bound}

\begin{lemma*}[restatement of \cref{lem:bern-first-reg-term-bound}]
    $
        \sum_{m=1}^\numintervals \bigl( \sum_{h=1}^{\Hm} \optimisticctg{m}(s_h^m) - \optimisticctg{m}(s_{h+1}^m) \bigr) \indgeventi{m}
        - 
        K \cdot \ctgopt(\sinit)
        \le
        2 \costbound |S| |A| \log T.
    $
\end{lemma*}

\begin{lemma}
\label{lem:bern-opt-val-bound}
Let $m$ be an interval. If $\geventi{m}$ holds then $\optimisticctg{m}(s) \leq \ctgopt(s) \leq \costbound$ for every $s \in S$.
\end{lemma}

\begin{proof}
    Denote by $\wt P$ the transition function computed by \cref{alg:bern-o-ssp} at the beginning of epoch $i(m)$, and by $\optimisticctg{}$ the cost-to-go with respect to $\wt P$.
    We claim that for every proper policy $\pi$ and state $s \in S$, $\optimisticctg{\pi}(s) \le \ctg{\pi}(s)$.
    Then, the lemma follows easily since $\optimisticctg{m}(s) \le \optimisticctg{\pi^\star}(s) \le \ctg{\pi^\star}(s) \le \costbound$.
    
    Indeed, let $s \in S$ and consider the Bellman equations of $\pi$ with respect to $P$:
    \begin{align*}
        \ctg{\pi}(s) 
        = 
        c(s,\pi(s)) + \sum_{s' \in S} P(s' \mid s,\pi(s)) \ctg{\pi}(s')
        \ge
        c(s,\pi(s)) + \sum_{s' \in S} \wt P(s' \mid s,\pi(s)) \ctg{\pi}(s'),
    \end{align*}
    where the inequality follows because $\wt P(s' \mid s,a) \le P(s' \mid s,a)$ for every $(s,a,s') \in S \times A \times S$. 
    This holds since $P$ is in the confidence set of \cref{eq:bernsteinconcentration} (as $\geventi{m}$ holds), and by the way $\wt P$ is computed in $\cref{alg:bern-o-ssp}$.
    Therefore, by \cref{lem:bertsekas-proper} we obtain that $\ctg{\pi} (s) \ge \optimisticctg{\pi}(s)$ for every $s \in S$ as required. 
\end{proof}

\begin{proof}[Proof of \cref{lem:bern-first-reg-term-bound}]
For every two consecutive intervals $m,m+1$, denoting $i = i(m)$, we have one of the following:
\begin{enumerate}[label=(\roman*)]
    \item If interval $m$ ended in the goal state then
    $
        \optimisticctg{i(m)}(s_{\Hm +1}^m) 
        =
        \optimisticctg{i(m)}(\ssink) = 0
    $ 
    and
    $
        \optimisticctg{i(m+1)}(s_1^{m+1}) 
        =
        \optimisticctg{i(m+1)}(\sinit).
    $
    Therefore, by \cref{lem:bern-opt-val-bound},
    \[
        \optimisticctg{i(m+1)}(s_1^{m+1}) \indgeventi{m+1}  
        - 
        \optimisticctg{i(m)}(s_{\Hm + 1}^m) \indgeventi{m} 
        = 
        \optimisticctg{i(m+1)}(\sinit) \indgeventi{m+1} 
        \leq 
        \ctgopt(\sinit).
    \]
    This happens at most $K$ times.
    
    \item If interval $m$ ended in an unknown state-action pair or since the cost reached $\costbound$, and we stay in the same epoch, then $i(m) = i(m+1) = i$ and $s_1^{m+1} = s_{\Hm +1}^m$. Thus
    \[
        \optimisticctg{i(m+1)}(s_1^{m+1}) \indgeventi{m+1}  
        - 
        \optimisticctg{i(m)}(s_{\Hm +1}^m) \indgeventi{m}
        =
        \optimisticctg{i}(s_1^{m+1}) \indgeventi{m} 
        - 
        \optimisticctg{i}(s_{\Hm +1}^m) \indgeventi{m} 
        = 
        0.
    \]
    
    \item If interval $m$ ended by doubling the visit count to some state-action pair, then we start a new epoch. Thus by \cref{lem:bern-opt-val-bound},
    \[
        \optimisticctg{i(m+1)}(s_1^{m+1}) \indgeventi{m+1}  
        - 
        \optimisticctg{i(m)}(s_{\Hm +1}^m) \indgeventi{m}
        \le 
        \optimisticctg{i+1}(s_1^{m+1}) \indgeventi{m+1} 
        \le 
        \costbound,
    \]
    This happens at most $2 |S| |A| \log \totaltime$ times.
\end{enumerate}

To conclude, we have
\begin{align*}
    \sum_{m=1}^\numintervals \Biggl( \sum_{h=1}^{\Hm} \optimisticctg{i(m)}(s_h^m) -  \optimisticctg{i(m)}(s_{h+1}^m) \Biggr) \indgeventi{m}
    -
    K \ctgopt(\sinit)
    &\le
    K \ctgopt(\sinit) + 2 \costbound |S| |A| \log \totaltime
    -
    K \ctgopt(\sinit)
    \\
    & =
    2 \costbound |S| |A| \log {T}. \qedhere
\end{align*}
\end{proof}

\subsubsection{Proof of \cref{lem:optimistic-val-func-diff-to-expected}}
\label{sec:proof-ptimistic-val-func-diff-to-expected}

\begin{lemma*}[restatement of \cref{lem:optimistic-val-func-diff-to-expected}]
    With probability at least $1 - \delta / 4$, the following holds for all $M = 1,2,\ldots$ simultaneously.
    \begin{align*}
        &\sum_{m=1}^\numintervals \Biggl( \sum_{h=1}^{\Hm} \optimisticctg{m}(s_{h+1}^m) - \sum_{s' \in S} \wt P_{m}(s' \mid s_h^m,a_h^m) \optimisticctg{m}(s') \Biggr) \indgeventi{m} \\
        & \qquad \le \sum_{m=1}^\numintervals \bbE \Biggl[ \Biggl( \sum_{h=1}^{\Hm} \optimisticctg{m}(s_{h+1}^m) - \sum_{s' \in S} \wt P_{m}(s' \mid s_h^m,a_h^m) \optimisticctg{m}(s') \Biggr) \indgeventi{m} \mid \trajconcat{m-1} \Biggr] 
        +
        3 \costbound \sqrt{\numintervals \log \frac{8 \numintervals}{\delta}}.
    \end{align*}
\end{lemma*}

\begin{proof}
    Consider the following martingale difference sequence $(X^m)_{m=1}^\infty$ defined by
    $$
        X^m 
        = 
        \sum_{h=1}^{\Hm} \bigl( \optimisticctg{m}(s_{h+1}^m) - \sum_{s' \in S} \wt P_{m}(s' \mid s_h^m,a_h^m) \optimisticctg{m}(s')  \bigr) \indgeventi{m}.
    $$
    The Bellman optimality equations of $\tilde{\pi}^m$ with respect to $\wt P_m$ (\cref{lem:bellman-optimistic}) obtains
    \begin{align*}
        |X^m|
        & = 
        \biggl| \biggl( \underbrace{\optimisticctg{m}(s_{\Hm +1}^m) - \optimisticctg{m}(s_{1}^m)}_{\le \costbound} + \underbrace{\sum_{h=1}^{\Hm} c(s_h^m,a_h^m)}_{\le 2 \costbound} \biggr) \indgeventi{m} \biggr|
        \le
        3 \costbound,
    \end{align*}
    where the inequality follows from \cref{lem:bern-opt-val-bound} and the fact that the total cost within each interval at most $2\costbound$ by construction.
    Therefore, we use anytime Azuma's inequality (\cref{thm:azuma}) to obtain that with probability at least $1-\delta / 4$:
    \begin{align*}
        \sum_{m=1}^M X^m 
        \le
        \sum_{m=1}^M \bbE \bigl[ X^m \mid \trajconcat{m-1} \bigr] 
        +
        3 \costbound \sqrt{M \log \frac{8 M}{\delta}}. \qquad 
        \label{eq:bern-azuma-deviations}
        \qedhere
    \end{align*}
\end{proof}

\subsubsection{Proof of \cref{lem:sum-bern-bounds}}
\label{sec:proof-sum-bern-bounds}

\begin{lemma*}[restatement of \cref{lem:sum-bern-bounds}]
    For every interval $m$ and time $h$, denote 
    $
        A_h^m 
        = 
        \tfrac{\log(|S| |A| N_+^{m}(s_{h}^m,a_{h}^m) / \delta)}{N_+^{m}(s_{h}^m,a_{h}^m)}.
    $
    Then,
    \begin{align*}
        &\bbE \Biggl[ \Biggl( \sum_{h=1}^{\Hm} \optimisticctg{m}(s_{h+1}^m) - \sum_{s' \in S} \wt P_{m}(s' \mid s_h^m,a_h^m) \optimisticctg{m}(s') \Biggr) \indgeventi{m} \mid \trajconcat{m-1} \Biggr] \\
        &\qquad \le 
        16 \cdot \bbE \Biggl[ 
        \sum_{h=1}^{\Hm} \sqrt{|S| \bbV_h^m A_h^m} \indgeventi{m} 
        \biggm| \trajconcat{m-1} \Biggr]
        + 
        272 \cdot \bbE \Biggl[ \sum_{h=1}^{\Hm} \costbound |S| A_h^m \indgeventi{m} 
        \biggm| \trajconcat{m-1} \Biggr],
    \end{align*}
    where $\bbV_h^m$ is the empirical variance defined as
    \[
        \bbV_h^m 
        = 
        \sum_{s' \in S^+} P(s' \mid s_{h}^m,a_{h}^m) \Biggl(\optimisticctg{m}(s') - \sum_{s'' \in S^+} P(s'' \mid s_{h}^m,a_{h}^m) \optimisticctg{m}(s'') \Biggr)^2.
    \]
\end{lemma*}

The next lemma gives a different interpretation to the confidence bounds of \cref{eq:bernsteinconcentration}, and will be useful in the proofs that follow.

\begin{lemma} 
\label{lem:relbernstein} 
Denote $A_h^m = \log (|S| |A| N_+^{m}(s,a) / \delta) / N_+^{m}(s,a)$.
When $\geventi{m}$ holds we have for any $(s,a,s') \in S \times A \times S^+$:
\[
    \bigl| P (s' \mid s,a) - \wt P_{m} (s' \mid s,a) \bigr|
    \le
    8 \sqrt{P(s' \mid s,a) A_h^m}
    +
    136 A_h^m.
\]
\end{lemma}

\begin{proof}
    Since $\geventi{m}$ holds we have for all $(s,a,s') \in S \times A \times S^+$ that
    \[
        \Bar P_{m} (s' \mid s,a) - P (s' \mid s,a)
        \le 
        4 \sqrt{\bar{P}_m(s' \mid s,a) A_h^m}
        + 
        28 A_h^m.
    \]
    This is a quadratic inequality in $\sqrt{\bar{P}_m(s' \mid s,a)}$. Using the fact that $x^2 \le a \cdot x + b$ implies $x \le a + \sqrt{b}$ with $a = 4 \sqrt{A_h^m}$ and $b = P (s' \mid s,a) + 28 A_h^m$, we have
    \[
        \sqrt{\bar P_m(s' \mid s,a)}
        \le
        4 \sqrt{A_h^m} + \sqrt{P(s' \mid s,a) + 28 A_h^m}
        \le
        \sqrt{P(s' \mid s,a)}  + 10 \sqrt{A_h^m},
    \]
    where we used the inequality $\sqrt{x + y} \le \sqrt{x} + \sqrt{y}$ that holds for any $x \ge 0$ and $y \ge 0$.
    Substituting back into \cref{eq:bernsteinconcentration} obtains
    \[
        \bigl| P (s' \mid s,a) - \Bar P_{m} (s' \mid s,a) \bigr|
        \le
        4 \sqrt{P(s' \mid s,a) A_h^m}
        +
        68 A_h^m.
    \]
    From a similar argument
    \[
        \bigl| \wt P_m (s' \mid s,a) - \Bar P_{m} (s' \mid s,a) \bigr|
        \le
        4 \sqrt{P(s' \mid s,a) A_h^m}
        +
        68 A_h^m.
    \]
    Using the triangle inequality finishes the proof.
\end{proof}

\begin{proof}[Proof of \cref{lem:sum-bern-bounds}]
Denote $X^m = \bigl( \sum_{h=1}^{\Hm} \optimisticctg{m}(s_{h+1}^m) - \sum_{s' \in S} \wt P_{m}(s' \mid s_h^m,a_h^m) \optimisticctg{m}(s') \bigr) \indgeventi{m}$, and $Z_h^m = \bigl(\optimisticctg{m}(s_{h+1}^m) - \sum_{s' \in S} P(s' \mid s_h^m,a_h^m) \optimisticctg{m}(s') \bigr) \indgeventi{m}$. Think of the interval as an infinite continuous stochastic process, and note that, conditioned on $\trajconcat{m-1}$,
$
    \bigl(Z_h^m \bigr)_{h=1}^\infty
$
is a martingale difference sequence w.r.t $(\traj{h})_{h=1}^\infty$, where $\traj{h}$ is the trajectory of the learner from the beginning of the interval and up to and including time $h$. This holds since, by conditioning on $\trajconcat{m-1}$, $\geventi{m}$ is determined and is independent of the randomness generated during the interval. 
Note that $\Hm$ is a stopping time with respect to $(Z_h^m)_{h=1}^\infty$ which is bounded by $2 \costbound / \cmin$. 
Hence by the optional stopping theorem
$
    \bbE [ \sum_{h=1}^{\Hm} Z_h^m \mid \trajconcat{m-1}] 
    = 
    0,
$
which gets us
\begin{align*}
    \bbE [X^m \mid \trajconcat{m-1} ]
    &=
    \bbE \Biggl[ \sum_{h=1}^{\Hm} \biggl( \optimisticctg{m}(s_{h+1}^m) - \sum_{s' \in S} \wt P_{m}(s' \mid s_h^m,a_h^m) \optimisticctg{m}(s')  \biggr) \indgeventi{m} \mid \trajconcat{m-1} \Biggr] \\
    &=
    \bbE \Biggl[ \sum_{h=1}^{\Hm} Z_h^m \mid \trajconcat{m-1} \Biggr]
    +
    \bbE \Biggl[ \sum_{h=1}^{\Hm} \sum_{s' \in S} \bigl( P(s' \mid s_h^m,a_h^m) - \wt P_{,}(s' \mid s_h^m,a_h^m) \bigr) \optimisticctg{m}(s') \indgeventi{m} \mid \trajconcat{m-1} \Biggr] \\
    &=
    \bbE \Biggl[ \sum_{h=1}^{\Hm} \sum_{s' \in S} \bigl( P(s' \mid s_h^m,a_h^m) - \wt P_{m}(s' \mid s_h^m,a_h^m) \bigr) \optimisticctg{m}(s') \indgeventi{m} \mid \trajconcat{m-1} \Biggr].
\end{align*}

Furthermore, we have
\begin{align*}
    &\bbE \Biggl[ 
    \sum_{h=1}^{\Hm} \sum_{s' \in S} \bigl(P(s' \mid s_{h}^m,a_{h}^m) - \wt P_m(s' \mid s_{h}^m,a_{h}^m)\bigr) \optimisticctg{m}(s') \indgeventi{m}
    \mid \trajconcat{m-1} \Biggr] 
    \\
    &\quad =
    \bbE \Biggl[ 
    \sum_{h=1}^{\Hm} \sum_{s' \in S^+} \biggl(P(s' \mid s_{h}^m,a_{h}^m) - \wt P_m(s' \mid s_{h}^m,a_{h}^m)\biggr)
    \biggl(\optimisticctg{m}(s') - \sum_{s'' \in S^+} P(s'' \mid s_{h}^m,a_{h}^m)  \optimisticctg{m}(s'') \biggr) \indgeventi{m}
    \mid \trajconcat{m-1} \Biggr] 
    \\
    &\quad \le
    \bbE \Biggl[
    8 \sum_{h=1}^{\Hm} \sum_{s' \in S^+} \sqrt{A_h^m P(s' \mid s_{h}^m,a_{h}^m) \Biggl(\optimisticctg{m}(s') - \sum_{s'' \in S^+} P(s'' \mid s_{h}^m,a_{h}^m) \optimisticctg{m}(s'') \Biggr)^2} \indgeventi{m} \mid \trajconcat{m-1} \Biggr]
    \\
    & \quad \qquad +
    \bbE \Biggl[136 \sum_{h=1}^{\Hm} \sum_{s' \in S^+} A_h^m
    \Biggl| \optimisticctg{m}(s') - \sum_{s'' \in S^+} P(s'' \mid s_{h}^m,a_{h}^m) \optimisticctg{m}(s'') \Biggr| \indgeventi{m}
    \mid \trajconcat{m-1} \Biggr]
    \\
    &\quad \le
    \bbE \Biggl[ 
    16 \sum_{h=1}^{\Hm} \sqrt{|S| \bbV_h^m A_h^m } \indgeventi{m}
    + 
    272 |S|  \costbound A_h^m  \indgeventi{m}
    \mid \trajconcat{m-1} \Biggr],
\end{align*}
where the first equality follows since $\optimisticctg{m}(\ssink) = 0$, and $P(\cdot \mid s_h^m, a_h^m)$ and $\wt P_i(\cdot \mid s_h^m, a_h^m)$ are probability distributions over $S^+$ whence $\sum_{s'' \in S^+} P(s'' \mid s_{h}^m,a_{h}^m)  \optimisticctg{m}(s'')$ does not depend on $s'$.
The first inequality follows from \cref{lem:relbernstein}, and the second inequality from Jensen's inequality, \cref{lem:bern-opt-val-bound}, $|S^+| \le 2 |S|$, and the definition of $\bbV_h^m$.
\end{proof}

\subsubsection{Proof of \cref{lem:bounded-V-h-i}}
\label{sec:proof-bounded-V-h-i}

\begin{lemma*}[restatement of \cref{lem:bounded-V-h-i}]
    For any interval $m$,
    $
        \bbE \bigl[\sum_{h=1}^{\Hm} \bbV_h^m \indgeventi{m} \mid \trajconcat{m-1} \bigr]
        \le 
        44 \costbound^2.
    $
\end{lemma*}

\begin{lemma}
\label{lem:bern-known-state}
Let $m$ be an interval and $(s,a)$ be a known state-action pair. 
If $\geventi{m}$ holds then for every $s' \in S^+$
\[
    \bigl| \wt P_m \bigl(s' \mid s, a \bigr)
    -
    P \bigl(s' \mid s, a \bigr) \bigr|
    \le
    \frac{1}{8} \sqrt{\frac{\cmin \cdot P \bigl(s' \mid s, a \bigr)}{|S| \costbound}} 
    + 
    \frac{\cmin}{4 |S| \costbound}.
\]
\end{lemma}

\begin{proof}
    By \cref{lem:relbernstein} we have that
    \begin{align*}
        \bigl| \wt P_m \bigl(s' \mid s, a \bigr)
        -
        P \bigl(s' \mid s, a \bigr) \bigr|
        & \le
        \csbern{m}{s}{a}{s'}
    \end{align*}
    which gives the required bound because $\log (x) / x$ is decreasing, and $(s,a)$ is a known state-action pair so $N_+^m(s,a) \ge 30000 \cdot \numvisitsuntilknownbern$.
\end{proof}

\begin{proof}[Proof of \cref{lem:bounded-V-h-i}]
Note that the first  state-action pair in the subinterval, $(s_{1}^m, a_{1}^m)$, might be unknown and that all state-action pairs that appear afterwards are known.
Thus, we bound
\begin{align*}
    \bbE \Biggl[\sum_{h=1}^{\Hm} \bbV_h^m \mid \trajconcat{m-1} \Biggr]
    &=
    \bbE \Biggl[\bbV_1^m \indgeventi{m} \mid \trajconcat{m-1} \Biggr] 
    +
    \bbE \Biggl[\sum_{h=2}^{\Hm} \bbV_h^m \indgeventi{m} \mid \trajconcat{m-1} \Biggr].
\end{align*}

The first summand is trivially bounded by $\costbound^2$ (\cref{lem:bern-opt-val-bound}).
We now upper bound 
$
    \bbE \bigl[\sum_{h=2}^{\Hm} \bbV_h^m \indgeventi{m} \mid \trajconcat{m-1} \bigr].
$
Denote $Z_h^m = \bigl(\optimisticctg{m}(s_{h+1}^m) - \sum_{s' \in S} P(s' \mid s_h^m,a_h^m) \optimisticctg{m}(s') \bigr) \indgeventi{m}$, and think of the interval as an infinite continuous stochastic process. Note that, conditioned on $\trajconcat{m-1}$,
$
    \bigl(Z_h^m \bigr)_{h=1}^\infty
$
is a martingale difference sequence w.r.t $(\traj{h})_{h=1}^\infty$, where $\traj{h}$ is the trajectory of the learner from the beginning of the interval and up to time $h$ and including. This holds since, by conditioning on $\trajconcat{m-1}$, $\geventi{m}$ is determined and is independent of the randomness generated during the interval. 
Note that $\Hm$ is a stopping time with respect to $(Z_h^m)_{h=1}^\infty$ which is bounded by $2 \costbound / \cmin$. 
Therefore, applying \cref{lem:martingalevariance} found below obtains
\begin{equation} 
\label{eq:timeslotvariance}
    \bbE \Biggl[\sum_{h=2}^{\Hm} \bbV_h^m \ind\{\geventi{m}\} \mid \trajconcat{m-1} \Biggr] 
    =
    \bbE \Biggl[\Biggl(\sum_{h=2}^{\Hm} Z_h^m \ind\{\geventi{m}\} \Biggr)^2  \mid \trajconcat{m-1} \Biggr].
\end{equation}

We now proceed by bounding $|\sum_{h=1}^{\Hm} Z_h^m |$ when $\geventi{m}$ occurs.
Therefore,
\begin{align}
    \Biggl| \sum_{h=2}^{\Hm} Z_h^m \Biggr|
    &=
    \Biggl|\sum_{h=2}^{\Hm} \optimisticctg{m}(s_{h+1}^m) - \sum_{s' \in S} P(s' \mid s_h^m,a_h^m) \optimisticctg{m}(s') \Biggr| 
    \nonumber \\
    &\le
    \Biggl|\sum_{h=2}^{\Hm} \optimisticctg{m}(s_{h+1}^m) - \optimisticctg{m}(s_{h}^m) \Biggr| 
    \label{eq:bern-var-bound-telescope} \\
    & \qquad + 
    \Biggl|\sum_{h=2}^{\Hm} \optimisticctg{m}(s_{h}^m) - \sum_{s' \in S} \wt P_m(s' \mid s_h^m,a_h^m) \optimisticctg{m}(s') \Biggr| 
    \label{eq:bern-var-bound-bellman} \\
    & \qquad +
    \Biggl| \sum_{h=2}^{\Hm} \sum_{s' \in S^+} \Bigl(  P(s' \mid s_h^m,a_h^m) - \wt P_m(s' \mid s_h^m,a_h^m) \Bigr) \Bigl(\optimisticctg{m}(s') - \sum_{s'' \in S^+} P(s'' \mid s_h^m,a_h^m) \optimisticctg{m}(s'') \Bigr) \Biggr|,
    \label{eq:bern-var-bound-err}
\end{align}
where \cref{eq:bern-var-bound-err} is given as $P(\cdot \mid s_h^m, a_h^m)$ and $\wt P_i(\cdot \mid s_h^m, a_h^m)$ are probability distributions over $S^+$, $\sum_{s'' \in S^+} P(s'' \mid s_h^m,a_h^m) \optimisticctg{m}(s'')$ is constant w.r.t $s'$, and $\optimisticctg{m}(\ssink) = 0$.

We now bound each of the three terms above individually.
\cref{eq:bern-var-bound-telescope} is a telescopic sum that is at most $\costbound$ on $\geventi{m}$ (\cref{lem:bern-opt-val-bound}).
For \cref{eq:bern-var-bound-bellman}, we use the Bellman  equations for $\tilde \pi^m$ on the optimistic model defined by the transitions $\wt P_m$ (\cref{lem:bellman-optimistic}) thus it is at most
$
    \sum_{h=2}^{H^m} c \bigl(s_h^m, a_h^m \bigr) \le 2 \costbound
$
(see text following \cref{lem:optimistic-val-func-diff-to-expected}).
For \cref{eq:bern-var-bound-err}, recall that all states-action pairs at times $h=2,\ldots,\Hm$ are known by definition of $\Hm$. Hence by \cref{lem:bern-known-state}, 
\begin{align*}
     &\Biggl| \sum_{s' \in S^+} \Bigl(\optimisticctg{m}(s') - \sum_{s'' \in S^+} P\bigl(s'' \mid s_h^m, a_h^m \bigr) \optimisticctg{m}(s'') \Bigr) 
    \Bigl( \wt P_m \bigl(s' \mid s_h^m, a_h^m \bigr) - P\bigl(s' \mid s_h^m, a_h^m\bigr) \Bigr) \Biggr| \nonumber \\
    &\qquad \leq 
    \frac{1}{8} \sum_{s' \in S^+}
    \sqrt{\frac{\cmin \cdot P \bigl(s' \mid s_h^m, a_h^m \bigr) \bigl(\optimisticctg{m}(s') - \sum_{s'' \in S^+} P\bigl(s'' \mid s_h^m, a_h^m \bigr) \optimisticctg{m}(s'') \bigr)^2}{|S| \costbound}} \\
    &\qquad\qquad + 
    \sum_{s' \in S^+}
    \frac{\cmin}{4 |S| \costbound} \cdot 
    \underbrace{\Bigl|\optimisticctg{m}(s') - \sum_{s'' \in S} P\bigl(s'' \mid s_h^m, a_h^m \bigr) \optimisticctg{m}(s'') \Bigr|}_{\le \costbound \text{ by \cref{lem:bern-opt-val-bound}}}
    \\
    &\qquad \leq 
    \frac{1}{4} \sqrt{\frac{\cmin \cdot \bbV_h^m}{\costbound}} 
    + 
    \frac{c \bigl(s_h^m, a_h^m \bigr)}{2}, 
    \tag{by Jensen's inequality, $\cmin \le c(s_h^m, a_h^m)$, $|S^+| \le 2 |S|$}
\end{align*}
and again by Jensen's inequality and that the total cost throughout the interval is at most $2 \costbound$, we have on $\geventi{m}$
\begin{align*}
    \sum_{h=2}^{\Hm} \frac{1}{4} \sqrt{\frac{\cmin \cdot \bbV_h^m}{\costbound}} 
    + 
    \frac{c \bigl(s_h^m, a_h^m \bigr)}{2}
    &\le
    \frac{1}{4} \sqrt{\underbrace{\Hm}_{\le 2\costbound/\cmin} \cdot \sum_{h=2}^{\Hm} \frac{\cmin \cdot \bbV_h^m}{\costbound}}
    + 
    \frac{1}{2} \underbrace{\sum_{h=2}^{\Hm} c \bigl(s_h^m, a_h^m \bigr)}_{\le 2 \costbound}
    \tag{Jensen's inequality} \\
    &\le
    \frac{1}{4} \sqrt{2 \sum_{h=2}^{\Hm} \bbV_h^m}
    + 
    \costbound.
\end{align*}

Plugging these bounds back into \cref{eq:timeslotvariance} gets us
\begin{align*}
    \bbE \Biggl[\sum_{h=2}^{\Hm} \bbV_h^m \indgeventi{m} \biggm| \trajconcat{m-1} \Biggr] \nonumber
    & \le
    \bbE \Biggl[
    \Biggl(
    4 \costbound 
    + 
    \frac{1}{4} \sqrt{2 \sum_{h=1}^{\Hm} \bbV_h^m \indgeventi{m}}
    \Biggr)^2 \biggm| \trajconcat{m-1} \Biggr] \\
    &\le
    32 \costbound^2
    +
    \frac{1}{4} \bbE \Biggl[\sum_{h=2}^{\Hm} \bbV_h^m \indgeventi{m} \biggm| \trajconcat{m-1} \Biggr], 
\end{align*}
where the last inequality is by the elementary inequality $(a+b)^2 \le 2(a^2 + b^2)$.
Rearranging gets us $\bbE \bigl[\sum_{h=2}^{\Hm} \bbV_h^m \indgeventi{m} \mid \trajconcat{m-1} \bigr] \le 43 \costbound^2$, and the lemma follows.
\end{proof}

\begin{lemma} \label{lem:martingalevariance}
    Let $(X_t)_{t=1}^\infty$ be a martingale difference sequence adapted to the filtration $(\calF_t)_{t=0}^\infty$. Let $Y_n = (\sum_{t=1}^n X_t)^2 - \sum_{t=1}^n \bbE[X_t^2 \mid \calF_{t-1}]$. Then $(Y_n)_{n=0}^\infty$ is a martingale, and in particular if $\tau$ is a stopping time such that $\tau \le c$ almost surely, then $\bbE[Y_\tau] = 0$.
\end{lemma}

\begin{proof}
    We first show that $(Y_n)_{n=1}^\infty$ is a martingale. Indeed,
    \begin{align*}
        \bbE[Y_n \mid \calF_{n-1}] 
        &=
        \bbE\Biggl[\Biggl(\sum_{t=1}^n X_t\Biggr)^2 - \sum_{t=1}^n \bbE[X_t^2 \mid \calF_{t-1}] \mid \calF_{n-1} \Biggr] \\
        &=
        \bbE\Biggl[\Biggl(\sum_{t=1}^{n-1} X_t\Biggr)^2 - 
        2 \Biggl(\sum_{t=1}^{n-1} X_t\Biggr) X_n 
        +
        X_n^2 - \sum_{t=1}^n \bbE[X_t^2 \mid \calF_{t-1}] \mid \calF_{n-1} \Biggr] \\
        &= 
        \Biggl(\sum_{t=1}^{n-1} X_t\Biggr)^2
        -
        2 \Biggl(\sum_{t=1}^{n-1} X_t\Biggr) \cdot 0
        + 
        \bbE[X_n^2 \mid \calF_{n-1}] 
        - 
        \sum_{t=1}^n \bbE[X_t^2 \mid \calF_{t-1}] 
        \tag{$\bbE[X_n \mid \calF_{n-1}] = 0$} \\
        &= 
        \Biggl(\sum_{t=1}^{n-1} X_t\Biggr)^2
        - 
        \sum_{t=1}^{n-1} \bbE[X_t^2 \mid \calF_{t-1}]
        =
        Y_{n-1}.
    \end{align*}
    
    We would now like to show that $\bbE[Y_\tau] = \bbE[Y_1] = 0$ using the optional stopping theorem. The latter holds since $\tau \le c$ almost surely and
    $
        \bbE[Y_1] = \bbE[X_1^2 - \bbE[X_1^2 \mid \calF_0]] = 0.
    $
\end{proof}

\subsubsection{Proof of \cref{lem:sum-bern-bounds-cont}}
\label{sec:proof-sum-bern-bounds-cont}

\begin{lemma*}[restatement of \cref{lem:sum-bern-bounds-cont}]
    With probability at least $1-\delta/4$, 
    \begin{align*}
        &\sum_{m=1}^\numintervals \bbE \Biggl[ \sum_{h=1}^{\Hm} \sum_{s' \in S} \bigl( P (s' \mid s_h^m,a_h^m) - \wt P_{m}(s' \mid s_h^m,a_h^m) \bigr) \optimisticctg{m}(s') \indgeventi{m} \mid \trajconcat{m-1} \Biggr] \\
        &\qquad \le
        614 \costbound \sqrt{\numintervals |S|^2 |A| \log^2 \frac{\totaltime |S| |A|}{\delta}}
        + 
        8160 \costbound |S|^2 |A| \log^2 \frac{\totaltime |S| |A|}{\delta}.
    \end{align*}
\end{lemma*}

\begin{proof}
Recall the following definitions:
\begin{align*}
    A_h^m 
    = 
    \frac{\log(|S| |A| N_+^{m}(s_{h}^m,a_{h}^m) / \delta)}{N_+^{m}(s_{h}^m,a_{h}^m)}.
    \qquad \qquad
    \bbV_h^m 
    = 
    \sum_{s' \in S^+} P(s' \mid s_{h}^m,a_{h}^m) \Biggl(\optimisticctg{m}(s') - \sum_{s'' \in S^+} P(s'' \mid s_{h}^m,a_{h}^m) \optimisticctg{m}(s'') \Biggr)^2.
\end{align*}
From \cref{lem:sum-bern-bounds} we have that 
\begin{align*}
    \bbE \Biggl[ 
    \sum_{h=1}^{\Hm} \sum_{s' \in S} \bigl(P(s' \mid s_{h}^m,a_{h}^m) 
    & - 
    \wt P_i(s' \mid s_{h}^m,a_{h}^m)\bigr) \optimisticctg{i}(s') \indgeventi{m}
    \mid \trajconcat{m-1} \Biggr]
    \\
    & \le
    \bbE \Biggl[ 
    16\sqrt{|S|} \sum_{h=1}^{\Hm}  \sqrt{ \bbV_h^m A_h^m} \indgeventi{m}
    + 
    272 \costbound |S| A_h^m \indgeventi{m}
    \mid \trajconcat{m-1} \Biggr].
\end{align*}
Moreover, by applying the Cauchy-Schwartz inequality twice, we get that
\begin{align*}
    \bbE \Biggl[ \sum_{h=1}^{\Hm}  \sqrt{\bbV_h^m A_h^m} \indgeventi{m} \biggm| \trajconcat{m-1} \Biggr]
    &\le 
    \bbE \Biggl[ \sqrt{\sum_{h=1}^{\Hm} \bbV_h^m \indgeventi{m}} \cdot \sqrt{\sum_{h=1}^{\Hm} A_h^m \indgeventi{m}} \biggm| \trajconcat{m-1} \Biggr] \\
    &\le
    \sqrt{\bbE \Biggl[ 
    \sum_{h=1}^{\Hm} A_h^m \indgeventi{m} \biggm| \trajconcat{m-1} \Biggr]} \cdot \sqrt{\bbE \Biggl[ 
    \sum_{h=1}^{\Hm} \bbV_h^m \indgeventi{m} \biggm| \trajconcat{m-1} \Biggr]} \\
    &\le
    7 \costbound
    \sqrt{\bbE \Biggl[ 
    \sum_{h=1}^{\Hm} A_h^m \indgeventi{m} \biggm| \trajconcat{m-1} \Biggr]}.
    \tag{\cref{lem:bounded-V-h-i}}
\end{align*}
We sum over all intervals to obtain
\begin{align*}
    &\sum_{m=1}^M \bbE \Biggl[ 
    \sum_{h=1}^{\Hm} \sum_{s' \in S} \bigl(P(s' \mid s_{h}^m,a_{h}^m) - \wt P_i(s' \mid s_{h}^m,a_{h}^m)\bigr) \optimisticctg{i}(s') \indgeventi{m}
    \mid \trajconcat{m-1} \Biggr]
    \le
    \\
    &\qquad \le
    112 \costbound \sum_{m=1}^\numintervals \sqrt{|S| \bbE \Biggl[ 
    \sum_{h=1}^{\Hm} A_h^m \indgeventi{m} \mid \trajconcat{m-1} \Biggr]}
    + 
    272 \costbound |S| \sum_{m=1}^\numintervals \bbE \Biggl[ \sum_{h=1}^{\Hm} A_h^m \indgeventi{m} \mid \trajconcat{m-1} \Biggr]
    \\
    &\qquad \le
    112 \costbound \sqrt{\numintervals |S| \sum_{m=1}^\numintervals \bbE \Biggl[ 
    \sum_{h=1}^{\Hm} A_h^m \indgeventi{m} \mid \trajconcat{m-1} \Biggr]}
    + 
    272 \costbound |S| \sum_{m=1}^\numintervals \bbE \Biggl[ \sum_{h=1}^{\Hm} A_h^m \indgeventi{m} \mid \trajconcat{m-1} \Biggr],
    \end{align*}
where the last inequality follows from Jensen's inequality.
We finish the proof using \cref{lem:bern-sum-visits} below.
\end{proof}

\begin{lemma}
\label{lem:bern-sum-visits}
With probability at least $1 - \delta / 4$, the following holds for $M=1,2,\ldots$ simultaneously.
\[
    \sum_{m=1}^\numintervals \bbE \Biggl[ \sum_{h=1}^{\Hm} A_h^m \indgeventi{m} \mid \trajconcat{m-1} \Biggr]
    \le
    O \biggl( |S| |A| \log^2 \frac{\totaltime |S| |A|}{\delta} \biggr).
\]
\end{lemma}

\begin{proof}
Define the infinite sequence of random variables: $X^m = \sum_{h=1}^{\Hm} A_h^m \indgeventi{m}$ for which $|X^m| \le 3 \log (|S||A| / \delta)$ due to \cref{lem:sum-inverse-visit-count-in-interval-bound} below.
We apply \cref{eq:anytime-bern-3} of \cref{lem:martingalte-multiplicative-bound} to obtain with probability at least $1 - \delta / 4$, for all $M =1,2,\ldots$ simultaneously
\begin{align*}
    \sum_{m=1}^\numintervals \bbE \bigl[ X^m \mid \trajconcat{m-1} \bigr]
    \le
    2 \sum_{m=1}^\numintervals X^m
    +
    12 \log \biggl(\frac{|S||A|}{\delta}\biggr) \log\biggl( \frac{8 M}{\delta} \biggr).
\end{align*}
Now, we bound the sum over $X^m$ by rewriting it as a sum over epochs:
\[
    \sum_{m=1}^\numintervals X^m
    \le
    \sum_{m=1}^\numintervals \sum_{h=1}^{\Hm} \frac{\log(|S| |A| N_+^{i}(s_{h}^m,a_{h}^m) / \delta)}{N_+^i(s_{h}^m,a_{h}^m)}
    \le
    \log \frac{|S| |A| \totaltime}{\delta} \sum_{s \in S} \sum_{a \in A} \sum_{i=1}^\numepochs \frac{n_i(s,a)}{N_+^i(s,a)},
\]
where $E$ is the last epoch. Finally, from \cref{lem:sum-inverse-visit-count-bound} below we have that for every $(s,a) \in S \times A$,
\[
    \sum_{i=1}^{\numepochs} \frac{n_i(s,a)}{N_+^i(s,a)}
    \le
    2 \log N_{\numepochs+1} (s,a)
    \le
    2 \log T.
\]
We now plugin the resulting bound for $\sum_{m=1}^M X^m$ and simplify the acquired expression by using $M \le T$.
\end{proof}

\begin{lemma} \label{lem:sum-inverse-visit-count-in-interval-bound}
    For any interval $m$,
    $
        | \sum_{h=1}^{\Hm} A_h^m | \le 3 \log (|S||A|/\delta).
    $
\end{lemma}

\begin{proof}
    Note that all state-action pairs $(s_h^m,a_h^m)$ (except the first one $(s_1^m,a_1^m)$) are known.
    Hence, for $h \ge 2$, $N_+^i(s_{h}^m,a_{h}^m) \ge 30000 \cdot \numvisitsuntilknownbern$. Therefore, since $\log(x)/x$ is decreasing and since $|S| \ge 2$ and $|A| \ge 2$ by assumption,
    \begin{align*}
        \sum_{h=1}^{\Hm} \frac{\log(|S| |A| N_+^{i}(s_{h}^m,a_{h}^m) / \delta)}{N_+^i(s_{h}^m,a_{h}^m)}
        & \le
        \frac{\log(|S| |A| N_+^{i}(s_1^m,a_1^m) / \delta)}{N_+^i(s_1^m,a_1^m)} 
        + 
        \sum_{h=2}^{\Hm} \frac{\log(|S| |A| N_+^{i}(s_{h}^m,a_{h}^m) / \delta)}{N_+^i(s_{h}^m,a_{h}^m)}
        \\
        & \le 
        \log(|S| |A| / \delta) + \frac{\cmin \Hm}{\costbound} \\
        & \le 
        \log(|S| |A| / \delta) + 2 
        \tag{$\Hm \le \tfrac{2 \costbound}{\cmin}$ by definition.} \\
        & \le 
        3 \log(|S| |A| / \delta). \qedhere
    \end{align*}
\end{proof}

\begin{lemma} \label{lem:sum-inverse-visit-count-bound}
For any sequence of integers $z_1,\dots,z_n$ with $0 \leq z_k \leq Z_{k-1} := \max \{ 1 , \sum_{i=1}^{k-1} z_i \}$ and $Z_0 = 1$, it holds that
\[
\sum_{k=1}^n \frac{z_k}{Z_{k-1}} \leq 2 \log Z_n.
\]
\end{lemma}

\begin{proof}
    We use the inequality $x \leq 2 \log (1+x)$ for every $0 \leq x \leq 1$ to obtain
    \[
    \sum_{k=1}^n \frac{z_k}{Z_{k-1}} 
    \le 
    2 \sum_{k=1}^n \log \biggl(1 + \frac{z_k}{Z_{k-1}} \biggr)
    = 
    2 \sum_{k=1}^n \log \frac{Z_{k-1} + z_k}{Z_{k-1}}
    = 
    2 \sum_{k=1}^n \log \frac{Z_k}{Z_{k-1}}
    =
    2 \log \prod_{k=1}^n \frac{Z_k}{Z_{k-1}}
    =
    2 \log Z_n. \qedhere
    \]
\end{proof}

\subsubsection{Proof of \cref{thm:bern-reg-bound}} \label{sec:proof-bern-reg-bound}

\begin{theorem*}[restatement of \cref{thm:bern-reg-bound}]
    Assume  that \cref{ass:c-min} holds. 
    With probability at least $1 - \delta$ the regret of \cref{alg:bern-o-ssp} is bounded as follows:
    \begin{align*}
        \regret
        =
        O \biggl( \costbound |S| \sqrt{ |A| K} \log \frac{K \costbound |S| |A|}{\delta \cmin } + \sqrt{\frac{\costbound^{3} |S|^4 |A|^2}{\cmin} } \log^2 \frac{K \costbound |S| |A|}{\delta \cmin } \biggr).
    \end{align*}
\end{theorem*}

\begin{proof}
    Let $C_M$ denote the cost of the learner after $M$ intervals.
    First, with probability at least $1 - \delta$, we have  \cref{lem:bern-num-time-steps,lem:sum-bern-bounds-cont,lem:optimistic-val-func-diff-to-expected} via a union bound.
    Now, as $\geventi{m}$ hold for all intervals, we have $\tregret{M} = R_M$ for any number of intervals $M$. Plugging in the bounds of \cref{lem:bern-first-reg-term-bound,lem:optimistic-val-func-diff-to-expected,lem:sum-bern-bounds-cont} into \cref{lem:bern-reg-decomp}, we have that for any number of intervals $M$: 
    \[
        C_M 
        =
        O \biggl(K \cdot \ctgopt(\sinit) 
        +
        \costbound \sqrt{\numintervals |S|^2 |A| \log^2 \frac{\totaltime |S| |A|}{\delta}}
        + 
        \costbound |S|^2 |A| \log^2 \frac{\totaltime |S| |A|}{\delta} \biggr).
    \]
    
We now plug in the bounds on $M$ and $T$ from \cref{obs:cost-bounds-intervals} into the bound above. First, we plug in the bound on $M$. As long as the $K$ episodes have not elapsed we have that
    $   
        M \le O \bigl(C_M / \costbound + K + 2 |S| |A| \log T + \frac{\costbound|S|^2 |A|}{\cmin} \log \frac{\costbound |S| |A|}{\delta \cmin} \bigr)
    $. 
    This gets after using the subadditivity of the square root to simplify the resulting expression,
    \begin{align*}
        C_M
        & =
        O \biggl(
        K \cdot \ctgopt(\sinit) 
        +
        \costbound \sqrt{K |S|^2 |A| \log^2 \frac{\totaltime |S| |A|}{\delta}}
        \\
        & \qquad \qquad +
        \sqrt{\costbound C_M |S|^2 |A| \log^2 \frac{\totaltime |S| |A|}{\delta}}
        + 
        \sqrt{\frac{\costbound^3 |S|^4 |A|^2}{\cmin} \log^4 \frac{T \costbound |S| |A|}{\cmin \delta}} 
        \biggr).
    \end{align*}
    
    From which, by solving for $C_M$ (using that $x \le a \sqrt{x} + b$ implies $x \le (a + \sqrt{b})^2$ for $a \ge 0$ and $b \ge 0$), and simplifying the resulting expression by applying $\ctgopt(\sinit) \le \costbound$ and our assumptions that $K \ge |S|^2 |A|$, $|S| \ge 2$, $|A| \ge 2$, we get that
    \begin{align} 
        C_M
        &=
        O \Biggl( \biggl( \sqrt{\costbound |S|^2 |A| \log^2 \frac{T |S| |A|}{\delta}} 
        \nonumber \\
        &\qquad +  
        \sqrt{K \cdot \ctgopt(\sinit)  
        +
        \costbound \sqrt{K |S|^2 |A| \log^2 \frac{T |S| |A|}{\delta}}
        +
        \sqrt{\frac{\costbound^3 |S|^4 |A|^2}{\cmin} \log^4 \frac{T \costbound |S| |A|}{\cmin \delta}}}
        \biggr)^2 \Biggr)
        \nonumber \\
        &=
        O \Biggl( \costbound |S|^2 |A| \log^2 \frac{T |S| |A|}{\delta} 
        \nonumber \\
        &\qquad +
        \sqrt{\costbound |S|^2 |A| \log^2 \frac{T |S| |A|}{\delta}} \cdot \sqrt{K \cdot \ctgopt(\sinit)  
        +
        \costbound \sqrt{K |S|^2 |A| \log^2 \frac{T |S| |A|}{\delta}}
        +
        \sqrt{\frac{\costbound^3 |S|^4 |A|^2}{\cmin} \log^4 \frac{T \costbound |S| |A|}{\cmin \delta}}} 
        \nonumber \\
        &\qquad +
        K \cdot \ctgopt(\sinit)  
        +
        \costbound \sqrt{K |S|^2 |A| \log^2 \frac{T |S| |A|}{\delta}}
        +
        \sqrt{\frac{\costbound^3 |S|^4 |A|^2}{\cmin} \log^4 \frac{T \costbound |S| |A|}{\cmin \delta}}
        \Biggr) 
        \nonumber \\
        &=
        O \Biggl( \costbound |S|^2 |A| \log^2 \frac{T |S| |A|}{\delta} 
        +
        \costbound \sqrt{K^{1/4} |S|^3 |A|^{3/2} \log^3 \frac{T |S| |A|}{\delta}}
        +
        \sqrt{\frac{\costbound^{5/2} |S|^4 |A|^2}{\cmin^{1/2}} \log^4 \frac{T \costbound |S| |A|}{\cmin \delta}}
        \nonumber \\
        &\qquad +
        K \cdot \ctgopt(\sinit)  
        +
        \costbound \sqrt{K |S|^2 |A| \log^2 \frac{T |S| |A|}{\delta}}
        +
        \sqrt{\frac{\costbound^3 |S|^4 |A|^2}{\cmin} \log^4 \frac{T \costbound |S| |A|}{\cmin \delta}}
        \Biggr) 
        \nonumber \\
        \nonumber \\
        &=
        O
        \Biggl(
        K \cdot \ctgopt(\sinit) 
        +
        \costbound \sqrt{K |S|^2 |A| \log^2 \frac{\totaltime |S| |A|}{\delta}} 
        + 
        \sqrt{\frac{\costbound^3 |S|^4 |A|^2}{\cmin} \log^4 \frac{T \costbound |S| |A|}{\cmin \delta}}
        \Biggr).
        \label{eq:bern-regret-with-t}
    \end{align}
    
    Note that in particular, by simplifying the bound above, we have
    $
        C_M
        =
        O \Bigl(\sqrt{\costbound^3 |S|^4 |A|^2 K \totaltime / \cmin \delta} \Bigr).
    $
    Next we combine this with the fact, stated in
    \cref{obs:cost-bounds-intervals} that $T\le C_M/ \cmin$. Isolating $T$ gets
    $
        \totaltime
        =
        O \Bigl(\tfrac{\costbound^3 |S|^4 |A|^2 K}{\cmin^3 \delta} \Bigr),
    $
    and plugging this bound back into \cref{eq:bern-regret-with-t} and simplifying gets us
    \begin{align*}
        C_M
        =
        O
        \biggl(
        K \cdot \ctgopt(\sinit) 
        +
        \costbound |S| \sqrt{|A| K \log^2 \frac{K \costbound |S| |A|}{\cmin \delta}}
        + 
        \sqrt{\frac{\costbound^3 |S|^4 |A|^2}{\cmin} \log^4 \frac{K \costbound |S| |A|}{\cmin \delta}}
        \biggr).
    \end{align*}
    
    Finally, we note that the bound above holds for any number of intervals $M$ as long as $K$ episodes do not elapse. As the instantaneous costs in the model are positive, this means that the learner must eventually finish the $K$ episodes from which we derive the bound for $\regret$ claimed by the theroem.
\end{proof}

\section{Lower Bound} \label{sec:lowerbound}

In this section we prove \cref{thm:lowerbound}.
At first glance, it is tempting to try and use the lower bound of \citet[Theorem 5]{AuerUCRL} on the regret suffered against learning average-reward MDPs by reducing any problem instance from an average-reward MDP to an instance of SSP.
However, it is unclear to us if such a reduction is possible, and if it is, how to perform it.\footnote{Even though a reduction in the reverse direction is fairly straight-forward in the unit-cost case \citep{tarbouriech2019noregret}.} We consequently prove the theorem here directly.

By Yao's minimax principle, in order to derive a lower bound on the learner's regret, it suffices to show a distribution over MDP instances that forces any deterministic learner to suffer a regret of $\Omega(\costbound \sqrt{|S| |A| K})$ in expectation.

To simplify our arguments, let us first consider the following simpler problem before considering the problem in its full generality.
Think of a simple MDP with two states: the initial state and a goal state.
The set of actions $A$ has a special action $a^\star$ chosen uniformly at random a-priori. 
Upon choosing the special action, the learner transitions to the goal state with probability $\approx 1 / \costbound$ and remains at $\sinit$ with the remaining probability.
Concretely $P(\ssink \mid a^\star) = 1 / \costbound$ and $P(\sinit \mid a^\star) = 1 - 1 / \costbound$, and for any other action $a \neq a^\star$ we have $P(\ssink \mid a) = (1 - \epsilon) / \costbound$ and $P(\sinit \mid a) = 1 - (1 - \epsilon) / \costbound$ for some $\epsilon \in (0,1/8)$.\footnote{For ease of notation and since there is only one state other than $\ssink$, we do not write this state as the origin state in the definition of the transition function.} The costs of all actions equal 1; i.e., $c(\sinit, a) = 1$ for all $a \in A$. Clearly, the optimal policy constantly plays $a^\star$ and therefore $\ctgopt(\sinit) = \costbound$.

Fix any deterministic learning algorithm, we shall now quantify the regret of the learner during a single episode in terms of the number of times that it chooses $a^\star$. 
Let $N_k$ denote the number of steps that the learner spends in $\sinit$ during episode $k$, and let $N^\star_k$ be the number of times the learner plays $a^\star$ at $\sinit$ during the episode. Note that $N_k$ is also the total cost that the learning algorithm suffered during episode $k$.
We have the following lemma.

\begin{lemma}
    \label{lem:lbregretrepresentation}
    $
        \bbE \bigl[N_k \bigr] - \ctgopt(\sinit)
        = 
        \epsilon \cdot \bbE \bigl[ N_k -  N^\star_k\bigr].
    $
\end{lemma}

\begin{proof}   
    Let us denote by $s_1,s_2,\ldots$ and $a_1,a_2,\ldots$ the sequences of states and actions observed by the learner during the episode.
    We have,
    \begin{align*}
        \bbE [N_k]
        &=
        \sum_{t=1}^\infty \Pr[s_t = \sinit] \\
        &=
        1 + \sum_{t=2}^\infty \Pr[s_t = \sinit] \\
        &=
        1 
        + 
        \sum_{t=2}^\infty \Pr[s_t = \sinit \mid s_{t-1} = \sinit, a_{t-1} = a^\star] \Pr[s_{t-1} = \sinit, a_{t-1} = a^\star] \\
        &\qquad +
        \sum_{t=2}^\infty \Pr[s_t = \sinit \mid s_{t-1} = \sinit, a_{t-1} \neq a^\star] \Pr[s_{t-1} = \sinit, a_{t-1} \neq a^\star] \\
        &=
        1 
        + 
        \sum_{t=2}^\infty \biggl(1 - \frac{1}{\costbound} \biggr) \Pr[s_{t-1} = \sinit, a_{t-1} = a^\star]
        +
        \sum_{t=2}^\infty \biggl(1 - \frac{1-\epsilon}{\costbound} \biggr) \Pr[s_{t-1} = \sinit, a_{t-1} \neq a^\star] \\
        &=
        1 
        + 
        \biggl(1 - \frac{1}{\costbound} \biggr) \sum_{t=1}^\infty \Pr[s_{t} = \sinit, a_{t} = a^\star]
        +
        \biggl(1 - \frac{1-\epsilon}{\costbound} \biggr) \sum_{t=1}^\infty \Pr[s_{t} = \sinit, a_{t} \neq a^\star] \\
        &=
        1 
        + 
        \biggl(1 - \frac{1}{\costbound} \biggr) \bbE [N_k^\star]
        +
        \biggl(1 - \frac{1-\epsilon}{\costbound} \biggr) \bbE[N_k - N^\star_k].
    \end{align*}
    Rearranging using $\ctgopt(\sinit) = \costbound$ gives the Lemma's statement.
\end{proof}

By   \cref{lem:lbregretrepresentation} the
 overall regret of the learner over $K$ episodes is: 
$
    \bbE[\regret] = \epsilon\cdot \bbE \bigl[ N - N^\star \bigr],
$
where $N = \sum_{k=1}^K N_k$ and $N^\star = \sum_{k=1}^K N_k^\star$. 
In words, the regret of the learner is $\epsilon$ times the expected number of visits to $\sinit$ in which the learner did not play $a^\star$.

In the remainder of the proof we lower bound $N$ in expectation and upper bound the expected value of $N^\star$.
To upper bound $N^\star$, we use standard techniques from lower bounds of multi-armed bandits \citep{auer2002nonstochastic} that bound the total variation distance between the distribution of the sequence of states traversed by the learner in the original MDP and that generated in a ``uniform MDP'' in which all actions are identical.
However, we cannot apply this argument directly since it requires $N^\star$ to be bounded almost surely, yet here $N^\star$ depends on the total length of all $K$ episodes which is unbounded in general. 
We fix this issue by looking only on the first $T$ steps (where $T$ is to be determined) and showing that the regret is large even in these $T$ steps.

Formally, we view the run of the $K$ episodes as a continuous process in which when the learner reaches the goal state we transfer it to  $\sinit$ (at no cost) and let it restart from there. Furthermore, we \emph{cap} the learning process to consist of exactly $T$ steps as follows.
If the $K$ episodes are completed before $T$ steps are elapsed, the learner remains in $\ssink$ (until completing $T$ steps) without suffering any additional cost, and otherwise we stop the learner after $T$ steps before it completes its $K$ episodes.
In this capped process, we denote the number of visits in $\sinit$ by $N_-$ and the number of times the learner played $a^\star$ in $\sinit$ by $N_-^\star$. 
We have
\begin{equation}
    \label{eq:cappedregret}
    \bbE[\regret] 
    \ge 
    \epsilon \cdot \bigl( \bbE \bigl[ N_- \bigr] - \bbE \bigl[ N^\star_- \bigr] \bigr).
\end{equation}

The number of visits to $\sinit$ under this capping is lower bounded by the following lemma.

\begin{lemma}
    \label{lem:cappedvisitlb}
    For any deterministic learner, if $T \ge 2 K \costbound$ then
    we have that
    $
        \bbE \bigl[ N_- \bigr]
        \ge
        K \costbound / 4.
    $
\end{lemma}

\begin{proof}
If the capped learner finished its $K$ episodes then
$N_- = N$. Otherwise, it visits the goal state less than $K$ times and therefore
 $N_- \ge T-K$. Hence
    $
        \bbE \bigl[ N_- \bigr]
        \ge
        \bbE \bigl[ \min\{T - K, N \} \bigr]
        \ge
        \sum_{k=1}^K \bbE \bigl[ \min\{T/K - 1, N_k\} \bigr].
    $
    Since $T \ge 2 K \costbound$, the lemma will follow if we show that $N_k \ge \costbound$ with probability at least $1/4$.
    We lower bound the probability that $N_k \ge \costbound$ by the probability of staying at $\sinit$ for $\costbound$ steps and  picking $a^\star$ in the first $\costbound-1$ steps. 
    Indeed, using $(1-1/x)^{x-1} \ge 1/e$ for $x \ge 1$, we get that
    $
        \Pr[N_k \ge \costbound] 
        \ge 
        \bigl(1-\frac{1}{\costbound}\bigr)^{\costbound - 1}
        \ge
        \frac{1}{4}.
    $
\end{proof}

We now introduce an additional distribution of the transitions which call  $\Pr_\text{unif}$.
$\Pr_\text{unif}$ is identical to $\Pr$ as  defined above, except that $P(\ssink \mid a) = (1-\epsilon) / \costbound$ for all actions $a.$
We denote expectations over $\Pr_\text{unif}$ by $\bbE_\text{unif}$.
The following lemma uses standard lower bound techniques used for multi-armed bandits (see, e.g., \citealp[Theorem 13]{AuerUCRL}) to bound the difference in the expectation of $N^\star_-$ 
when the learner plays in $\Pr$ compared to when it plays in $\Pr_\text{unif}$.

\begin{lemma}
    \label{lem:astarub}
    For any deterministic learner we have that    
    $
        \bbE \bigl[N^\star_- \bigr]
        \le
        \bbE_\text{unif} \bigl[N^\star_- \bigr]
        +
        \epsilon T \sqrt{ \bbE_\text{unif} [N^\star_- ]/B}.
    $
\end{lemma}

\begin{proof}
    Fix any deterministic learner.
    Let us denote by $s^{(t)}$ the sequence of states observed by the learner up to time $t$ and including.
    Now, as $N_-^\star \le T$ and the fact that $N_-^\star$ is a function of $s^{(T)}$,
    $
        \bbE \bigl[N_-^\star \bigr]
        \le
        \bbE_\text{unif} \bigl[N^\star_- \bigr]
        + 
        T \cdot \TV{\Pr_\text{unif}[s^{(T)}]}{\Pr[s^{(T)}]},
    $
    and Pinsker's inequality yields
    \begin{equation}
        \label{eq:pinsker}
        \TV{\Pr_\text{unif}[s^{(T)}]}{\Pr[s^{(T)}]}
        \le 
        \sqrt{\frac{1}{2} \KL{\Pr_\text{unif}[s^{(T)}]}{\Pr[s^{(T)}]}}.
    \end{equation}
    Next, the chain rule of the KL divergence obtains
    \[
        \KL{\Pr_\text{unif}[s^{(T)}]}{\Pr[s^{(T)}]}
        =
        \sum_{t=1}^T \sum_{s^{(t-1)}} \Pr_\text{unif}[s^{(t-1)}] \cdot \KL{\Pr_\text{unif}[s_t \mid s^{(t-1)}]}{\Pr[s_t \mid s^{(t-1)}]}.
    \]
    
    Observe that at any time, since the learning algorithm is deterministic, the learner chooses an action given $s^{(t-1)}$ regardless of whether $s^{(t-1)}$ was generated under $\Pr$ or under $\Pr_\text{unif}$. Thus, the $\KL{\Pr_\text{unif}[s_t \mid s^{(t-1)}]}{\Pr[s_t \mid s^{(t-1)}]}$ is zero if $a_{t-1} \neq a_\star$, and otherwise
    \begin{align*}
        \KL{\Pr_\text{unif}[s_t \mid s^{(t-1)}]}{\Pr[s_t \mid s^{(t-1)}]}
        &=
        \sum_{s \in S} \Pr_\text{unif}[s_{t} \mid s_{t-1} = \sinit, a_{t-1} = a^\star] \log \frac{\Pr_\text{unif}[s_{t} \mid s_{t-1} = \sinit, a_{t-1} = a^\star]}{\Pr[s_{t} \mid s_{t-1} = \sinit, a_{t-1} = a^\star]} \\
        &=
        \frac{1-\epsilon}{\costbound} \cdot \log(1-\epsilon) + \biggl(1-\frac{1-\epsilon}{\costbound} \biggr) \log \biggl(1 + \frac{\epsilon}{ \costbound -1} \biggr) \\
        &\le
        \frac{\epsilon^2}{ \costbound -1}.
        \tag{using $\log(1+x) \le x$ for all $x > 0$}
    \end{align*}
    
    Plugging the above back into \cref{eq:pinsker} and using $\costbound \ge 2$ gives the lemma.
\end{proof}

In the following result, we combine the lemma above with standard techniques from lower bounds of multi-armed bandits (see \citealp[Thm. 5]{AuerUCRL} for example).

\begin{theorem}
    \label{thm:twostatelb}
    Suppose that $\costbound \ge 2$, $\epsilon \in (0,\frac18)$ and $|A| \ge 16$.
    For the problem described above we have that
    \[
        \bbE[\regret] \ge \epsilon K \costbound
        \biggl( \frac{1}{8} - 2 \epsilon \sqrt{\frac{2K}{|A|}} \biggr).
    \]
\end{theorem}

\begin{proof}[Proof of \cref{thm:twostatelb}]
    Note that as under $\Pr_\text{unif}$ the transition distributions are identical for all actions, we have that
    
    \begin{equation}
        \label{eq:unifsumactions}
        \sum_{a \in A \mid a^\star = a} \bbE_\text{unif} \bigl[ N^\star_- \bigr] 
        =
        \bbE_\text{unif} \Biggl[ \sum_{a \in A \mid a^\star = a} N^\star_- \Biggr] 
        = 
        \bbE_\text{unif} \bigl[ N_- \bigr]
        \le 
        T.
    \end{equation}
    
    Suppose that $a^\star$ is sampled uniformly at random before the game starts. 
    Denote the probability and expectation with respect to the distribution induced by a specific choice of $a^\star = a$ by $\Pr_a$ and $\bbE_a$ respectively.
    Then for $T = 2K\costbound$,
    \begin{align*}
        \bbE[\regret]
        &=
        \frac{1}{|A|} \sum_{a \in A} \bbE_a[\regret] \\
        &\ge
        \frac{1}{|A|} \sum_{a \in A} \bbE_a[N_- - N^\star_-]
        \tag{\cref{eq:cappedregret}} \\
        &\ge
        \frac{1}{|A|} \sum_{a \in A \mid a_\star = a} \biggl(\frac{K \costbound}{4} - \bbE_\text{unif} [N^\star_-] - \epsilon T \sqrt{\frac{\bbE_\text{unif} [N^\star_-]}{\costbound}} \biggr)
        \tag{\cref{lem:cappedvisitlb,lem:astarub}} \\
        &\ge
        \frac{K \costbound}{4} - \frac{1}{|A|} \sum_{a \in A \mid a_\star = a} \bbE_\text{unif} [N^\star_-] - \epsilon T \sqrt{\frac{1}{\costbound} \cdot \frac{1}{|A|} \sum_{a \in A \mid a_\star = a} \bbE_\text{unif} [N^\star_-]}
        \tag{Jensen's inequality} \\
        &\ge
        \frac{K \costbound}{4} - \frac{T}{|A|} - \epsilon T \sqrt{\frac{T}{\costbound |A|}}
        \tag{\cref{eq:unifsumactions}} \\
        &=
        \epsilon \biggl( \frac{K \costbound}{4} - \frac{2K \costbound}{|A|} - 2\epsilon K \costbound \sqrt{\frac{2 K \costbound}{|A| \costbound}} \biggr) \\
        &= \epsilon K \costbound \biggl( \frac{1}{4} - \frac{2}{|A|} - 2\epsilon \sqrt{\frac{2 K }{|A| }} \biggr).
    \end{align*}
    The theorem follows from $|A| \ge 16$ and by rearranging.
\end{proof}

\begin{proof}[Proof of \cref{thm:lowerbound}]
    Consider the following MDP. 
    Let $S$ be the set of states disregarding $\ssink$. 
    The initial state is sampled uniformly at random from $S$. 
    Each $s \in S$ has its own special action $a^\star_s$. 
    The transition distributions are defined 
    $
        P(\ssink \mid a^\star_s, s) 
        =
        1 / \costbound,
    $ 
    $
        P(s \mid a^\star_s, s) 
        =
        1 - 1 / \costbound,
    $
        and 
    $
        P(\ssink \mid a, s) = (1 - \epsilon) / \costbound,
    $ 
    $
        P(s \mid a, s) = 1 - (1 - \epsilon) / \costbound
    $ 
    for any other action $a \in A \backslash \{a^\star_s\}$.
    
    Note that for each $s \in S$, the learner is faced with a simple problem as the one described above from which it cannot learn about from other states $s' \neq s$. 
    Therefore, we can apply \cref{thm:twostatelb} for each $s \in S$ separately and lower bound the learner's expected regret the sum of the regrets suffered at each $s \in S$, which would depend on the number of times $s \in S$ is drawn as the initial state. Since the states are chosen uniformly at random there are many states (constant fraction) that are chosen $\Theta(K/|S|)$ times. Summing the regret bounds of \cref{thm:twostatelb} over only these states and choosing $\epsilon$ appropriately gives the sought-after bound.
    
    Denote by $K_s$ the number of episodes that start in each state $s \in S$. 
    \begin{align}
        \bbE[\regret ]
        \ge 
        \sum_{s \in S} 
        \bbE
        \Biggl[\epsilon K_s \costbound \biggl( \frac{1}{8} - 2 \epsilon \sqrt{\frac{2K_s}{|A|}} \biggr) \Biggr]
        =
        \frac{\epsilon K \costbound}{8} - 2 \epsilon^2 \costbound \sqrt{\frac{2}{|A|}} \sum_{s \in S} \bbE[K_s^{3/2}]. 
        \label{eq:regretlb}
    \end{align}
    Taking expectation over the initial states and applying Cauchy-Schwartz inequality gives
    \[
        \sum_{s \in S} \bbE \bigl[K_s^{3/2} \bigr]
        \le
        \sum_{s \in S} \sqrt{\bbE [K_s]} \sqrt{\bbE [K_s^2]}
        =
        \sum_{s \in S} \sqrt{\bbE [K_s]} \sqrt{\bbE [K_s]^2 + \bbV[K_s]}
        =
        \sum_{s \in S} \sqrt{\frac{K}{|S|}} \sqrt{\frac{K^2}{|S|^2} + \frac{K (|S|-1)}{|S|^2}}
        \le
        K \sqrt{\frac{2K}{|S|}},
    \]
    where we have used the expectation and variance formulas of the Binomial distribution.
    The lower bound is now given by applying the inequality above in \cref{eq:regretlb} and choosing $\epsilon = \frac{1}{64} \sqrt{|A| |S|/K}$.
\end{proof}

\section{Concentration inequalities}

\begin{theorem}[Anytime Azuma] \label{thm:azuma}
    Let $(X_n)_{n=1}^\infty$ be a martingale difference sequence with respect to the filtration $(\calF_n)_{n=0}^\infty$ such that $|X_n| \le B$ almost surely. 
    Then with probability at least $1-\delta$,
    \[
        \Biggl|\sum_{i=1}^n X_i \Biggr| \le B \sqrt{n \log \frac{2 n}{\delta}}, \qquad \forall n \ge 1.
    \]
\end{theorem}

\begin{theorem}[\citealp{weissman2003inequalities}] \label{thm:weissman}
    Let $p(\cdot)$ be a distribution over $m$ elements, and let $\bar{p}_t(\cdot)$ be the empirical distribution defined by $t$ iid samples from $p(\cdot)$.
    Then, with probability at least $1 - \delta$,
    \[
        \bigl\lVert \bar{p}_t(\cdot) - p(\cdot) \bigr\rVert_1
        \le
        2 \sqrt{\frac{ m \log \frac{1}{\delta}}{t}}.
    \]
\end{theorem}

\begin{theorem}[Anytime Bernstein] \label{thm:bernstein}
    Let $(X_n)_{n=1}^\infty$ be a sequence of i.i.d.\ random variables with expectation $\mu$.
    Suppose that $0 \le X_n \le B$ almost surely. Then with probability at least $1-\delta$, the following holds for all $n \ge 1$ simultaneously:
    \begin{align}
        \label{eq:anytime-bern-1}
        &\Biggl| \sum_{i=1}^n (X_i - \mu) \Biggr|
        \le
        2 \sqrt{B \mu n \log \frac{2n}{\delta}} + B \log \frac{2n}{\delta}. \\
        \label{eq:anytime-bern-2}
        &\Biggl| \sum_{i=1}^n (X_i - \mu) \Biggr|
        \le
        2 \sqrt{B \sum_{i=1}^n X_i \log \frac{2n}{\delta}} + 7 B \log \frac{2n}{\delta}.
    \end{align}
\end{theorem}

\begin{proof}
    Fix some $n \ge 1$. By Bernstein's concentration inequality (see for example, \citealp[Corollary A.3]{cesa2006prediction}), we have with probability at least $1-\tfrac{\delta}{2n^2}$ that \cref{eq:anytime-bern-1} holds. By a union bound, the inequality holds with probability at least $1-\delta$ for all $n \ge 1$ simultaneously.
    
    To show \cref{eq:anytime-bern-2}, note that in particular we have 
    \[
        \mu \cdot n - \sum_{i=1}^n X_i
        \le
        2 \sqrt{B \mu n \log \frac{2n}{\delta}} 
        + 
        B \log \frac{2n}{\delta}
    \]
    that is a quadratic inequality in $\mu$.
    This implies that 
    \[
        \sqrt{\mu}
        \le
        \sqrt{\frac{1}{n} \sum_{i=1}^n X_i}
        +
        3 \sqrt{\frac{B \log \frac{2 n}{\delta}}{n}}.
    \]
    Plugging this inequality back into the RHS of \cref{eq:anytime-bern-1} gets us \cref{eq:anytime-bern-2}.
\end{proof}

\begin{lemma} \label{lem:martingalte-multiplicative-bound}
    Let $(X_n)_{n=1}^\infty$ be a sequence of random variables with expectation adapted to the filtration $(\calF_n)_{n=0}^\infty$.
    Suppose that $0 \le X_n \le B$ almost surely. Then with probability at least $1-\delta$, the following holds for all $n \ge 1$ simultaneously:
    \begin{equation}
        \label{eq:anytime-bern-3}
        \sum_{i=1}^n \bbE[X_i \mid \calF_{i-1} ]
        \le
        2 \sum_{i=1}^n X_i + 4 B \log \frac{2n}{\delta}.
    \end{equation}
\end{lemma}

\begin{proof}
    For all $n \ge 1$, we have
    \begin{align*}
        \bbE[e^{-X_n / B} \mid \calF_{n-1}]
        &\le
        \bbE\biggl[1 - \frac{X_n}{B} + \frac{X_n^2}{2B^2} \Bigm| \calF_{n-1}\biggr]
        \tag{$e^{-x} \le 1-x+\tfrac{x^2}{2}$ for all $x \ge 0$} \\
        &\le
        1 - \frac{\bbE[X_n \mid \calF_{n-1}]}{B} + \frac{\bbE[ X_n \mid \calF_{n-1}]}{2 B} 
        \tag{$X_n \le B$} \\
        &=
        1 - \frac{\bbE[X_n \mid \calF_{n-1}]}{2B} \\
        &\le
        e^{-\bbE[X_n \mid \calF_{n-1}] / 2B}.
        \tag{$1-x \le e^{-x}$ for all $x$}
    \end{align*}
    
    Hence, fix some $n \ge 1$, then
    \begin{align*}
        &\bbE\Biggl[ \exp\biggl(\frac{1}{B} \sum_{i=1}^n \biggl( \frac{1}{2} \bbE[X_i \mid \calF_{i-1}] - X_i \biggr) \biggr) \Biggr] \\
        &\qquad =
        \bbE\Biggl[ \exp\biggl(\frac{1}{B} \sum_{i=1}^{n-1} \biggl( \frac{1}{2} \bbE[X_i \mid \calF_{i-1}] - X_i \biggr) \biggr)
        \cdot 
        \underbrace{\bbE\Biggl[ \exp\biggl(\frac{1}{B}  \biggl( \frac{1}{2} \bbE[X_n \mid \calF_{n-1}] - X_n \biggr) \biggm| \calF_{n-1} \Biggr]}_{\le 1} \Biggr] \\
        &\qquad \le 
        \bbE\Biggl[ \exp\biggl(\frac{1}{B} \sum_{i=1}^{n-1} \biggl( \frac{1}{2} \bbE[X_i \mid \calF_{i-1}] - X_i \biggr) \biggr) \Biggr] \\
        &\qquad \le 1.
        \tag{by repeating the last argument inductively.}
    \end{align*}
    
    Therefore,
    \begin{align*}
        \Pr \Biggl[ \sum_{i=1}^n \biggl( \frac{1}{2} \bbE[X_i \mid \calF_{i-1}] - X_i \biggr) > 2 B \log \frac{2n}{\delta} \Biggr]
        &\le
        \Pr \Biggl[ \exp\biggl(\frac{1}{B} \sum_{i=1}^n \biggl( \frac{1}{2} \bbE[X_i \mid \calF_{i-1}] - X_i \biggr) \biggr) > \frac{2n^2}{\delta} \Biggr] \\
        &\le
        \bbE \Biggl[ \exp\biggl(\frac{1}{B} \sum_{i=1}^n \biggl( \frac{1}{2} \bbE[X_i \mid \calF_{i-1}] - X_i \biggr) \biggr) \Biggr] \cdot \frac{\delta}{2n^2}
        \tag{Markov inequality} \\
        &\le
        \frac{\delta}{2n^2}.
    \end{align*}
    Hence the above holds for all $n \ge 1$ via a union bound which provides the lemma.
\end{proof}